\documentclass{article}

\usepackage[preprint]{neurips_2025}

\usepackage[utf8]{inputenc} 
\usepackage[T1]{fontenc}    
\usepackage{url}            
\usepackage{booktabs}       
\usepackage{amsfonts}       
\usepackage{nicefrac}       
\usepackage{microtype}      
\usepackage{xcolor}         
 \usepackage{amsmath} 

\usepackage{url}
\usepackage{amsthm}
\usepackage{amsmath}
\usepackage{graphicx}
\usepackage{booktabs}
\usepackage{xcolor}
\usepackage{wrapfig}

\definecolor{icmlblue}{rgb}{0.21,0.49,0.74}
\usepackage[pagebackref,breaklinks,colorlinks,citecolor=icmlblue]{hyperref}

\definecolor{electricindigo}{rgb}{0.44, 0.0, 1.0}
\usepackage{multirow}

\definecolor{deblue}{RGB}{11,132,147}
\definecolor{ocra}{RGB}{204, 119, 34}
\usepackage{enumitem}
\usepackage{tikz}
\newcommand{\fcircle}[2][red,fill=red]{\tikz[baseline=-0.5ex]\draw[#1,radius=#2] (0,0.03) circle ;}

\definecolor{deblue}{RGB}{11,132,147}
\definecolor{ocra}{RGB}{204, 119, 34}
\definecolor{electricindigo}{rgb}{0.44, 0.0, 1.0}
\usepackage[capitalize]{cleveref}
\crefname{section}{Sec.}{Secs.}
\Crefname{section}{Section}{Sections}
\Crefname{table}{Table}{Tables}
\crefname{table}{Tab.}{Tabs.}
\definecolor{indigo(web)}{rgb}{0.29, 0.0, 0.51}
\usepackage{multirow}
\usepackage{xcolor}
\definecolor{darkblue}{RGB}{40,40,85}
\definecolor{babyblue}{rgb}{0.54, 0.81, 0.94}
\definecolor{pearDark}{HTML}{2980B9}
\definecolor{pearDarker}{HTML}{1D2DEC}
\usepackage[capitalize]{cleveref}
\crefname{section}{Sec.}{Secs.}
\Crefname{section}{Section}{Sections}
\Crefname{table}{Table}{Tables}
\crefname{table}{Tab.}{Tabs.}
\usepackage{fontawesome5}
\usepackage{multirow}
\usepackage{colortbl}
\usepackage{tabulary}
\usepackage{etoolbox}
\usepackage{tikz}
\usepackage{pifont}
\usepackage{booktabs}
\usepackage{longtable}
\usepackage{array}
\usepackage{amsmath,amssymb}

\theoremstyle{plain}
\newtheorem{theorem}{Theorem}[section]

\newtheorem{lemma}[theorem]{Lemma}

\theoremstyle{definition}
\newtheorem{definition}[theorem]{Definition}

\theoremstyle{remark}

\title{CATO: Charted Attention for Neural PDE Operators}

\author{%
  Chun-Wun Cheng\\
    DAMTP\\ University of Cambridge \\
  \texttt{cwc56@cam.ac.uk} 
  \And
    Sifan Wang\\
    Institute for Foundations of Data Science\\ Yale University \\
  \texttt{sifan.wang@yale.edu} \\
  \And
    Carola-Bibiane~Sch\"onlieb\\
    DAMTP\\ University of Cambridge \\
  \texttt{cbs31@cam.ac.uk} \\
  \And
    Angelica I. Aviles-Rivero\thanks{Corresponding author.} \\
    Yau Mathematical Sciences Center\\ Tsinghua University\\
  \texttt{aviles-rivero@tsinghua.edu.cn} \\
}

\begin{document}

\maketitle

\begin{abstract}
Neural operators have emerged as powerful data-driven solvers for PDEs, offering substantial acceleration over classical numerical methods. However, existing transformer-based operators still face critical challenges when modeling PDEs on complex geometries: directly processing over massive mesh points is computationally expensive, while operating in raw discretization coordinates may obscure the intrinsic geometry where physical interactions are more naturally expressed. To address these limitations, we introduce the Charted Axial Transformer Operator (CATO), a geometry-adaptive and derivative-aware neural operator for PDEs on general geometries. Instead of applying attention directly in the physical coordinate system, CATO learns a continuous latent chart that maps mesh coordinates into a learned chart space, where chart-conditioned axial attention efficiently captures long-range dependencies with reduced computational cost. In addition, CATO introduces a derivative-aware physics loss for steady-state PDEs that jointly supervises solution values, mesh-consistent gradients, and an auxiliary flux-like field, improving physical fidelity and reducing oversmoothing. We further provide a theoretical approximation result showing that, under a favorable chart, charted axial attention can represent low-rank axial solution operators with controlled error, and that small chart perturbations induce bounded approximation degradation. CATO achieves the best performance across all evaluated datasets, yielding an average improvement of approximately 26.76\% over the strongest competing baselines while reducing the number of parameters by 81.98\%. These results highlight the effectiveness of learning geometry-adaptive charts and derivative-aware physical supervision for accurate and efficient PDE operator learning.


\end{abstract}

\addtocontents{toc}{\protect\setcounter{tocdepth}{-1}}
\section{Introduction}

Many real-world phenomena, including turbulence and atmospheric circulation, are governed by partial differential equations (PDEs) \cite{debnath2012linear}. Classical numerical methods, such as finite element and spectral methods \cite{solin2005partial,costa2004spectral}, can produce highly accurate solutions, but they are often computationally expensive and therefore poorly suited to real-time prediction or many-query scenarios. This computational bottleneck has motivated growing interest in data-driven alternatives. The increasing availability of high-fidelity simulation data, together with advances in deep learning, has enabled the development of learned surrogate solvers that trade a modest loss in accuracy for substantial gains in computational efficiency. Unlike classical solvers, which typically solve each new PDE instance from scratch, learned surrogates amortize computation costs across many related problem settings.

Neural operators \cite{lu2019deeponet,wen2022u,li2023geometry,wu2024transolver, bryutkin2024hamlet,cheng2025mamba,wang2025fourier, cheng2025pde} have emerged as a promising data-driven alternative by learning mappings between function spaces directly from data. They enable fast inference and generalization across resolutions and have been successfully applied to weather forecasting \cite{pathak2022fourcastnet,leinonen2024modulated}, medical imaging \cite{hadramy2026noir,jatyani2025coarse}, and scientific modeling \cite{herde2024poseidon,zhou2024unisolver}. Transformer-based approaches \cite{cao2021choose,liu2022ht,li2022transformer,hao2023gnot,xiao2023improved,wu2024transolver,zhou2026saot, wang2024cvit} have further improved the modeling of nonlocal interactions, but remain challenged by computational cost and the difficulty of capturing meaningful geometric structure on large meshes.
A central limitation of existing methods is that they often operate directly in discretization coordinates, which may be poorly aligned with the intrinsic geometry of the underlying physical process. Consequently, the operator can appear unnecessarily complex, making it more difficult to learn compact and efficient representations. 

\textit{We hypothesize that this coordinate mismatch is a central bottleneck in neural operator learning.} To address this, we learn a geometry-adaptive coordinate chart before applying attention, transforming the operator into a representation that is easier to approximate. Concretely, we propose the \textbf{Charted Axial Transformer Operator (CATO)}, which maps the physical domain into a continuous chart space and performs attention in this adapted geometry.
On grids and structured meshes, where chart coordinates provide an ordered factorization, CATO applies axial attention along coordinate directions with a lightweight local operator, capturing both long-range dependencies and local structure without incurring the cost of full attention. On unstructured point clouds, we instead use CATO-PC, a topology-aware variant that replaces axial attention with KNN-based local aggregation and global irregular attention.
In addition, CATO incorporates derivative-aware supervision for steady-state PDEs by jointly predicting the solution and a flux representation, improving physical fidelity and stability. More generally, many PDEs on curved domains admit a low-rank or separable structure when expressed in a coordinate system aligned with the physics. By learning this coordinate system end-to-end, CATO shifts the burden from the attention mechanism to a simple learned embedding. This is fundamentally different from prior work that either fixes the coordinate system or compresses tokens without reparameterizing the geometry.
Our contributions are summarized as follows:

\fcircle[fill=deblue]{2pt} We identify \emph{coordinate mismatch} as a fundamental bottleneck in neural operator learning, where models must simultaneously learn geometry and solution structure. We show that adapting the coordinate system can reduce the effective complexity of the operator.

\fcircle[fill=deblue]{2pt} We propose the \textbf{Charted Axial Transformer Operator (CATO)}, which learns a continuous coordinate chart $\Phi_{\mathrm{chart}}$ and applies axial attention in this space, transforming a general nonlocal operator into an approximately separable (axial low-rank) form that can be efficiently approximated.

\fcircle[fill=deblue]{2pt} We establish that CATO provably approximates charted axial low-rank operators with explicit error bounds yielding both approximation guarantees and stability to chart perturbations.

\fcircle[fill=deblue]{2pt} Across six PDE benchmarks, CATO achieves an average $26.76\%$ error reduction (up to $52.74\%$), while using $81.98\%$ fewer parameters and training up to $3.5\times$ faster than prior methods.

\section{Related Work}
\paragraph{Neural PDE Solvers.}
Classical numerical methods (finite difference, finite element, spectral) remain the gold standard for accuracy, but their computational cost prohibits real‑time and many‑query applications. Early deep learning approaches, such as Physics‑Informed Neural Networks (PINNs) \cite{raissi2019physics}, incorporate PDE residuals directly into the loss, enabling unsupervised training but often suffering from training instability and spectral bias.
Operator learning offers an alternative paradigm: learn a mapping between function spaces directly from paired data. DeepONet \cite{lu2019deeponet} first demonstrated this idea. FNO \cite{li2020fourier} introduced global convolution in the spectral domain, achieving resolution invariance. Subsequent works improved expressivity and efficiency: U‑FNO \cite{wen2022u} and U‑NO \cite{rahman2022u} added multi‑scale paths; Geo‑FNO \cite{li2023fourier} learned deformations to handle irregular geometries; GINO \cite{li2023geometry} extended to 3D point clouds; LSM \cite{wu2023solving} leveraged latent spectral representations; WMT \cite{gupta2021multiwavelet} used wavelet decompositions. Despite their success, most of these methods assume regular grids or rely on hand‑crafted deformations; none adapt the coordinate system dynamically for attention.

\paragraph{Transformer-Based Neural Operators.}
Due to the fact that self-attention can be viewed as a learnable nonlocal integral operator, transformers have been an essential stride into neural PDE solving. Specific techniques like the Galerkin Transformer \cite{cao2021choose}, which implemented kernels in a linear attention without softmax, and models such as HT-Net \cite{liu2022ht}, OFormer \cite{li2022transformer}, GNOT \cite{hao2023gnot}, ONO \cite{xiao2023improved}, and FactFormer \cite{li2023scalable} used hierarchical, linear, orthogonal, or factorized approaches to provide a better trade-off between accurate long-range interaction modeling while maintaining computational efficiency. These approaches showed that attention-based architectures can be successful in learning PDE solution operators. SAOT \cite{zhou2026saot} combines Fourier attention for global patterns with Wavelet attention for local, high-frequency details. Transolver \cite{wu2024transolver} uses discrete slices to form physical attention, while our method maps the original mesh into a continuous chart space with axial attention.

\textbf{Comparison with Existing Methods.} CATO differs fundamentally from the above methods. While Transolver compresses physical tokens, it still operates in raw coordinates; SAOT mixes Fourier and wavelet attention but does not reparameterize geometry; OFormer and GNOT rely on fixed positional encodings. CATO instead learns a continuous geometry chart \(\Phi_{\mathrm{chart}}\) and applies axial attention in that adapted space – reducing complexity to \(O(HW(H+W))\) and aligning attention with the PDE’s natural low‑rank structure. Additionally, CATO introduces a derivative‑aware loss that supervises both solution values and a gradient‑like flux, improving sharpness on distorted meshes – a feature absent in all prior transformer‑based operators. We provide theoretical guarantees that learning a chart reduces the effective operator complexity and that small chart errors cause only linear degradation.

\section{Methodology}

\paragraph{Problem statement.}
In neural operator learning, we consider operator approximation on a two-dimensional structured mesh of resolution $H \times W$, with $N = HW$ nodes. For each sample, let
$\mathbf{X} = \{\mathbf{x}_{ij}\}_{i=1,j=1}^{H,W} \in \mathbb{R}^2$,
denote the physical coordinates of the mesh nodes, and let
$\mathbf{F} = \{\mathbf{f}_{ij}\}_{i,j=1}^{H,W} \in \mathbb{R}^{d_f}$,
denote optional node-wise auxiliary inputs, such as coefficients, source terms, or other field descriptors. The objective is to learn a solution operator
$\mathcal{G}_\theta : (\mathbf{X}, \mathbf{F}) \mapsto \mathbf{u}$,
where the target scalar field is given by $\mathbf{u} = \{u_{ij}\} \in \mathbb{R}^{H \times W}$. The model predicts the scalar solution field
$\hat{\mathbf{u}} \in \mathbb{R}^{B \times N \times 1}$.
During training, it also produces an auxiliary vector field
$\hat{\mathbf{q}} \in \mathbb{R}^{B \times N \times 2}$,
which is supervised using the spatial gradient of the target field. Thus, this auxiliary head can be interpreted as a gradient-like flux proxy.

\begin{figure}[t!]
    \centering
    \includegraphics[width=1\linewidth]{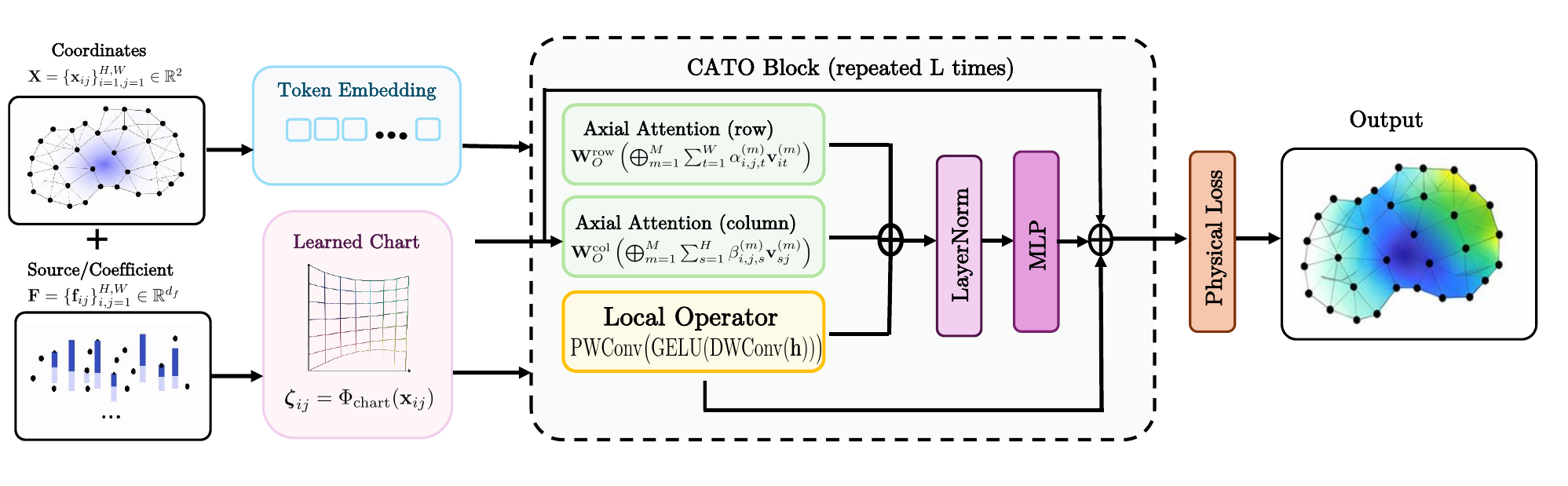}
    \caption{CATO architecture overview. Coordinates and source features are embedded with a learned chart, processed by repeated CATO blocks combining axial attention and local operators, and trained with a physics-informed loss to predict the output field. }
    \label{teaser}
\end{figure}

\subsection{Charted Axial Transformer Operator (CATO) Block}
\label{CATO block}
For each node, the physical coordinate and optional auxiliary features are concatenated: $\mathbf{z}_{ij}^{\mathrm{in}} =
[\mathbf{x}_{ij}, \mathbf{f}_{ij}] \ \text{if } d_f > 0,\;
\mathbf{x}_{ij} \ \text{otherwise}$. These inputs are lifted into a higher latent space of dimension $C$ by a two-layer MLP:
$\mathbf{h}_{ij}^{(0)} 
= \Phi_{\mathrm{pre}}(\mathbf{z}_{ij}^{\mathrm{in}})
= \mathbf{W}_2 \, \sigma(\mathbf{W}_1 \mathbf{z}_{ij}^{\mathrm{in}} + \mathbf{b}_1) + \mathbf{b}_2$,. The initial hidden representation can be written as: $\mathbf{H}^{(0)} \in \mathbb{R}^{B \times H \times W \times C}$.

\paragraph{Learnable geometry chart.} The physical grid is typically constructed as a discrete representation of the computational domain, rather than being induced by the PDE itself. Its primary role is to encode the domain geometry and boundary structure, not necessarily the intrinsic coordinate system in which the solution operator is most naturally expressed. Consequently, the raw Cartesian coordinates $(x,y)$ may be poorly aligned with the dominant directions of variation in the solution, particularly on curved or non-uniform meshes. They may also encode redundant geometric information, entangle relevant and irrelevant directions for attention, and force the model to compensate for mesh distortion before learning the underlying operator. This motivates performing attention in a learned, geometry-adapted coordinate system, rather than assuming that the physical mesh coordinates are aligned with the intrinsic geometry of the solution operator.

We introduce a learned chart that maps each physical coordinate to a continuous latent 2D chart space. 
$\boldsymbol{\zeta}_{ij} = (\xi_{ij}, \eta_{ij}) = \Phi_{\mathrm{chart}}(\mathbf{x}_{ij}), \text{with}$
$\Phi_{\mathrm{chart}}(\mathbf{x})
=
\tanh\!\left(
\mathbf{V}_2 \, \mathrm{SiLU}(\mathbf{V}_1 \mathbf{x} + \mathbf{c}_1) + \mathbf{c}_2
\right)$.
Hence, $(\xi_{ij}, \eta_{ij}) \in [-1,1]^2.$ $\xi_{ij}$ is used for row attention while $\eta_{ij}$ is used for column attention. We do not require $\Phi_{\mathrm{chart}}$ to be globally invertible; instead, it is used as a learned continuous coordinate system for positional encoding and attention.

\paragraph{Continuous rotary positional encoding (RoPE)} Discrete positional encoding only encodes token features while neglecting the relative distance of tokens. However, closer points in PDEs will have a stronger influence, indicating that relative distances are an important factor in solving PDEs. To mitigate this limitation, we use continuous RoPE, which not only retains token information but also preserves relative distance relationships. The axial attention layers use continuous RoPE, where the positional variable is not a discrete token index but a real-valued chart coordinate. 

For each head dimension pair
$r = 0,1,\dots,\frac{d_h}{2}-1$,
 define the angular frequency $\omega_r = \theta^{-2r/d_h}$,
where $\theta > 0$ is the RoPE base parameter.  The rotary transform matrix is defined as:
\[
R_r(p)
\begin{bmatrix}
z_{2r} \\
z_{2r+1}
\end{bmatrix}
=
\begin{bmatrix}
\cos(\omega_r p) & -\sin(\omega_r p) \\
\sin(\omega_r p) & \cos(\omega_r p)
\end{bmatrix}
\begin{bmatrix}
z_{2r} \\
z_{2r+1}
\end{bmatrix}.
\]
Applying this over all channel pairs gives $\widetilde{\mathbf{q}} = R(p)\mathbf{q},
\qquad
\widetilde{\mathbf{k}} = R(p)\mathbf{k}$. Then we can define the attention score as: $\left( R(p_i) q_i \right)^T \left( R(p_j) k_j \right) = q_i^T R(p_j - p_i) k_j$ which contains both token feature and relative distance features. In addition, the input $p$ is a coordinate, which is a continuous input of the position. Continuous position functions impose a smooth geometric structure on attention. Nearby positions change by small rotations, which often matches the real structure of sequences better than a purely index-based view.

\paragraph{Charted axial self-attention.}
After obtaining the learned chart, we apply multi-head self-attention separately along the row and column directions. The row-wise and column-wise attention outputs are then summed to form the final axial attention representation.

Specifically, let the hidden representation at a node be $\mathbf{h}_{ij} \in \mathbb{R}^C$. Queries, keys, and values are computed as $\mathbf{q}_{ij} = \mathbf{W}_Q \mathbf{h}_{ij}, \qquad
\mathbf{k}_{ij} = \mathbf{W}_K \mathbf{h}_{ij}, \qquad
\mathbf{v}_{ij} = \mathbf{W}_V \mathbf{h}_{ij}$.
With $M$ attention heads and head dimension $d_h = C/M$, these are split as $\mathbf{q}_{ij}^{(m)},\ \mathbf{k}_{ij}^{(m)},\ \mathbf{v}_{ij}^{(m)} \in \mathbb{R}^{d_h},
\qquad m = 1,\dots,M$.

We first compute the row attention.
For a fixed row $i$, the tokens $\{\mathbf{h}_{ij}\}_{j=1}^{W}$ form a 1D sequence. The horizontal chart coordinate $\xi_{ij}$ is used in RoPE: $\widetilde{\mathbf{q}}_{ij}^{(m)} = R(\xi_{ij})\mathbf{q}_{ij}^{(m)},
\qquad
\widetilde{\mathbf{k}}_{ij}^{(m)} = R(\xi_{ij})\mathbf{k}_{ij}^{(m)}$.
The row-attention can be computed as:
$\mathrm{Attn}_{\mathrm{row}}(\mathbf{h})_{ij}
=
\mathbf{W}_O^{\mathrm{row}}
\left(
\bigoplus_{m=1}^{M}
\sum_{t=1}^{W}
\alpha_{i,j,t}^{(m)} \mathbf{v}_{it}^{(m)}
\right)$, where $\alpha_{i,j,t}^{(m)}$ is the attention weight that is computed by softmax.
Similarly, we compute the column attention output as:
$\mathrm{Attn}_{\mathrm{col}}(\mathbf{h})_{ij}
=
\mathbf{W}_O^{\mathrm{col}}
\left(
\bigoplus_{m=1}^{M}
\sum_{s=1}^{H}
\beta_{i,j,s}^{(m)} \mathbf{v}_{sj}^{(m)}
\right)$,
where $\beta_{i,j,s}^{(m)}$ are the corresponding softmax-normalized column-attention weights. The final output is the sum of row and column outputs: $\mathcal{A}(\mathbf{h}, \boldsymbol{\zeta})
=
\mathrm{Attn}_{\mathrm{row}}(\mathbf{h}; \xi)
+
\mathrm{Attn}_{\mathrm{col}}(\mathbf{h}; \eta)$.

To complement the nonlocal attention, we further introduce a local depthwise operator:
$\mathcal{L}(\mathbf{h})
=
\mathrm{PWConv}
\bigl(
\mathrm{GELU}(
\mathrm{DWConv}(\mathbf{h})
)
\bigr)$,
where $\mathrm{DWConv}$ denotes a depthwise $k \times k$ convolution and $\mathrm{PWConv}$ a $1 \times 1$ pointwise convolution. It acts as a learned local stencil operator.

We now define the CATO block as follows. Given hidden state $\mathbf{H}^{(\ell)}$, we compute
\begin{equation}
\widetilde{\mathbf{H}}^{(\ell)}
=
\mathbf{H}^{(\ell)}
+
\mathcal{A}\!\left(
\mathrm{LN}(\mathbf{H}^{(\ell)}), \boldsymbol{\zeta}
\right)
+
\mathcal{L}\!\left(
\mathrm{LN}(\mathbf{H}^{(\ell)})
\right).
\end{equation}

A second residual update is then applied: $\mathbf{H}^{(\ell+1)}
=
\widetilde{\mathbf{H}}^{(\ell)}
+
\mathrm{MLP}\!\left(
\mathrm{LN}(\widetilde{\mathbf{H}}^{(\ell)})
\right)$, where MLP denotes a feed-forward network. We then stack L blocks. After $L$ CATO blocks, a final layer normalization is applied: $\mathbf{H}^{(L)} \leftarrow \mathrm{LN}(\mathbf{H}^{(L)})$.

The final latent state is mapped to two outputs. The scalar solution prediction is
$\hat{u}_{ij} = \mathbf{w}_u^\top \mathbf{h}_{ij}^{(L)} + b_u$.
The auxiliary vector output is
$\hat{\mathbf{q}}_{ij}
=
\mathbf{W}_q \mathbf{h}_{ij}^{(L)} + \mathbf{b}_q,
\qquad
\hat{\mathbf{q}}_{ij} \in \mathbb{R}^2$.
Therefore, the model predicts both a scalar field $\hat{u}$ and a gradient-like flux field $\hat{\mathbf{q}}$.

For inputs without a canonical grid structure (e.g., point clouds), the row–column factorisation required by axial attention is not defined. In this setting, we retain the learned chart as the core representation, but replace axial attention with a geometry-aware attention operator defined on local neighborhoods. This results in a point-cloud variant (CATO-PC) that preserves the chart-based formulation while adapting the interaction mechanism to the input topology.
\subsection{Physical Loss}
Instead of predicting only $u$ (pressure or scalar field), we also predict a gradient proxy as an auxiliary output. This tends to improve sharp features, reduce oversmoothing, and stabilize learning when data is limited.

We construct it as follows. Let the coordinate at node $(i,j)$ be $\mathbf{x}_{ij} = (x_{ij}, y_{ij})$.
Define centered differences
$\Delta_i u_{ij} = u_{i+1,j} - u_{i-1,j},
\qquad
\Delta_j u_{ij} = u_{i,j+1} - u_{i,j-1}$,
and $\Delta_i \mathbf{x}_{ij}
=
\mathbf{x}_{i+1,j} - \mathbf{x}_{i-1,j},
\qquad
\Delta_j \mathbf{x}_{ij}
=
\mathbf{x}_{i,j+1} - \mathbf{x}_{i,j-1}$.
Let $\Delta_i \mathbf{x}_{ij} = (a,b),
\qquad
\Delta_j \mathbf{x}_{ij} = (c,d)$,
and rewrite it in linear-system form:
\[
\begin{bmatrix}
\Delta_i u_{ij} \\
\Delta_j u_{ij}
\end{bmatrix}
\approx
\begin{bmatrix}
a & b \\
c & d
\end{bmatrix}
\begin{bmatrix}
u_x \\
u_y
\end{bmatrix}_{ij}.
\]
We can obtain the solution by solving the linear system and we get:
$u_x
=
\frac{
\Delta_i u_{ij}\, d - \Delta_j u_{ij}\, b
}{
ad-bc
},
\qquad
u_y
=
\frac{
-\Delta_i u_{ij}\, c + \Delta_j u_{ij}\, a
}{
ad-bc
}$.
This gives the discrete gradient approximation $(u_x, u_y)$. $|ad - bc| > 0$ ensures the system is non-singular; otherwise, the local mesh directions are linearly dependent and the gradient is not uniquely defined. Physical supervision enforces consistency in both function values and spatial derivatives, leading to improved fidelity of local structures and reduced smoothing bias.

\paragraph{Training objective.}

The training loss combines value accuracy, gradient matching, auxiliary flux supervision, and consistency between the flux head and the gradient implied by the predicted scalar field.
The total loss is defined as follows:
$\mathcal{L}
=
\mathcal{L}_{\mathrm{val}}
+
\lambda_g \mathcal{L}_{\mathrm{grad}}
+
\lambda_f \mathcal{L}_{\mathrm{flux}}
+
\lambda_c \mathcal{L}_{\mathrm{cons}}$,
where \(\lambda_g\), \(\lambda_f\), and \(\lambda_c\) control the relative contributions of the gradient, flux, and consistency terms.

First, the value loss measures the relative $L^2$ error between the predicted and reference scalar fields:
$\mathcal{L}_{\mathrm{val}}
=
\frac{1}{B}
\sum_{b=1}^{B}
\frac{
\left\|
\hat{\mathbf{u}}^{(b)}-\mathbf{u}^{(b)}
\right\|_2
}{
\left\|
\mathbf{u}^{(b)}
\right\|_2+\varepsilon
}$, where $B$ is the batch size and $\varepsilon>0$ ensures numerical stability.

To incorporate derivative information, we reconstruct gradients on the structured mesh as
$\nabla \mathbf{u}=\mathrm{Grad}(\mathbf{u},\mathbf{X}),
\qquad
\nabla \hat{\mathbf{u}}=\mathrm{Grad}(\hat{\mathbf{u}},\mathbf{X})$,
where $\mathbf{X}$ denotes the mesh coordinates. The gradient-matching loss is then defined as
$\mathcal{L}_{\mathrm{grad}}
=
\frac{1}{BN}
\sum_{b=1}^{B}\sum_{n=1}^{N}
\left\|
\nabla \hat{\mathbf{u}}_{b,n}
-
\nabla \mathbf{u}_{b,n}
\right\|_2^2$.

The auxiliary vector head $\hat{\mathbf{q}}$ is directly supervised by the target gradient through the flux loss:
$\mathcal{L}_{\mathrm{flux}}
=
\frac{1}{BN}
\sum_{b=1}^{B}\sum_{n=1}^{N}
\left\|
\hat{\mathbf{q}}_{b,n}
-
\nabla \mathbf{u}_{b,n}
\right\|_2^2$ .
To enforce compatibility between the scalar and auxiliary outputs, we further introduce the consistency loss:
$\mathcal{L}_{\mathrm{cons}}
=
\frac{1}{BN}
\sum_{b=1}^{B}\sum_{n=1}^{N}
\left\|
\hat{\mathbf{q}}_{b,n}
-
\nabla \hat{\mathbf{u}}_{b,n}
\right\|_2^2$ .
Together, these objectives provide field-level, derivative-level, and consistency supervision, promoting accurate and spatially coherent predictions.
\paragraph{Overall design.} As show in figure \ref{teaser}, the overall architecture of CATO is designed as a geometry-adaptive neural operator for solving PDEs on general domains. The model first embeds the input mesh coordinates and optional physical features into a latent representation. A learned chart module then maps the original physical coordinates into a continuous chart space, where stacked CATO blocks apply axial self-attention to efficiently capture long-range dependencies. Each block also includes a lightweight local operator to model nearby spatial interactions. Finally, the processed representation is decoded into the target solution field and an auxiliary gradient-like flux field, improving both prediction accuracy and physical consistency.
\subsection{Theoretical underpinning}
Why should learning a geometry chart help? A raw Cartesian grid often does not align with the intrinsic directions of a PDE solution—for example, flow along a curved pipe or around an airfoil. In such cases, the solution operator may be approximately separable along coordinate directions when expressed in a suitable coordinate system, yet appear complex in the original $(x,y)$ coordinates.
CATO’s core hypothesis is that, by learning a coordinate chart $\zeta = \Phi_{\mathrm{chart}}(x)$ and applying axial attention in this chart space, the operator can be transformed into a representation that is significantly easier to approximate.
We now formalise this intuition. For the theoretical analysis, we consider a  CATO block with setting: dropout is set to zero, LayerNorm is replaced by the identity, and the local depthwise branch is deactivated. For clarity, we state the results for a scalar input field $f \in \mathbb{R}^{H \times W}$; the extension to vector-valued fields follows analogously.

Given a chart
$\zeta_{ij}=(\xi_{ij},\eta_{ij})=\Phi_{\mathrm{chart}}(x_{ij})\in K\subset[-1,1]^2$,
and a  one-block of CATO acts as
$H^{(0)}_{ij}=\Phi_{\mathrm{pre}}(x_{ij},f_{ij})\in\mathbb{R}^C$,
$\widetilde H = H^{(0)} + A(H^{(0)},\zeta),
\qquad
H^{(1)} = \widetilde H + \mathrm{MLP}(\widetilde H)$,
followed by a linear readout $\mathcal N_\Theta(f,X)_{ij}=w_{\mathrm{out}}^\top H^{(1)}_{ij}+b_{\mathrm{out}}$. Then we have the following definition and lemma.

\begin{definition}[Charted axial low-rank operator]
\label{def}
Let $B_M:=\{f\in\mathbb{R}^{H\times W}:\|f\|_2\le M\}$.
We say that an operator
$\widetilde{\mathcal G}_\Phi:B_M\to\mathbb{R}^{H\times W}$
is \emph{$(R_\xi,R_\eta,\varepsilon_{\mathrm{rk}})$-charted axial low-rank}
(with respect to the chart $\zeta$) if there exist continuous functions $a_r,b_r,c_s,d_s,\ell:K\to\mathbb{R},
\qquad
r=1,\dots,R_\xi,\quad s=1,\dots,R_\eta$,
and an operator $\mathcal R$ such that
$\widetilde{\mathcal G}_\Phi=\mathcal T_\zeta+\mathcal R$,
where$(\mathcal T_\zeta f)_{ij}
=
\sum_{r=1}^{R_\xi}
a_r(\zeta_{ij})
\Big(\frac1W\sum_{t=1}^W b_r(\zeta_{it})f_{it}\Big)
+
\sum_{s=1}^{R_\eta}
c_s(\zeta_{ij})
\Big(\frac1H\sum_{p=1}^H d_s(\zeta_{pj})f_{pj}\Big)
+
\ell(\zeta_{ij})f_{ij}$,
and
$\|\mathcal R f\|_2\le \varepsilon_{\mathrm{rk}}\|f\|_2
\qquad
\text{for all } f\in B_M$.
\end{definition}
\begin{lemma}[Neural realization of charted axial finite-rank operators]
\label{lem:axial-realization}
Let $\mathcal T_\zeta:B_M\to\mathbb{R}^{H\times W}$ be given by
$(\mathcal T_\zeta f)_{ij}
=
\sum_{r=1}^{R_\xi}
a_r(\zeta_{ij})
\Big(\frac1W\sum_{t=1}^W b_r(\zeta_{it})f_{it}\Big)
+
\sum_{s=1}^{R_\eta}
c_s(\zeta_{ij})
\Big(\frac1H\sum_{p=1}^H d_s(\zeta_{pj})f_{pj}\Big)
+
\ell(\zeta_{ij})f_{ij}$,
where $a_r,b_r,c_s,d_s,\ell$ are continuous on $K$.
Then for every $\varepsilon_{\mathrm{nn}}>0$, there exist a hidden width $C$
and parameters of a one-block core CATO with $R_\xi$ row heads and $R_\eta$
column heads such that $\sup_{f\in B_M}\|\mathcal N_\Theta(f,X)-\mathcal T_\zeta f\|_2
\le
\varepsilon_{\mathrm{nn}}$.
\end{lemma}
This result shows that the CATO block can approximate any finite-rank charted axial operator to arbitrary accuracy. The proof constructs row and column attention heads that perform the required directional averaging operations across the chart coordinates.
\begin{lemma}[Lipschitz stability with respect to chart perturbations]
\label{lem:chart-stability}
Let $\mathcal T_\zeta$ be as in Lemma~\ref{lem:axial-realization}, and assume
in addition that the coefficient functions are bounded and Lipschitz:
$\|a_r\|_\infty\le A_r,\quad
\|b_r\|_\infty\le B_r,\quad
\|c_s\|_\infty\le C_s,\quad
\|d_s\|_\infty\le D_s,\quad
\|\ell\|_\infty\le L_0$,
and
$\operatorname{Lip}(a_r)\le L_{a_r},\quad
\operatorname{Lip}(b_r)\le L_{b_r},\quad
\operatorname{Lip}(c_s)\le L_{c_s},\quad
\operatorname{Lip}(d_s)\le L_{d_s},\quad
\operatorname{Lip}(\ell)\le L_\ell$.
Let another chart $\widehat\zeta_{ij}\in K$ satisfy
$\max_{i,j}\|\widehat\zeta_{ij}-\zeta_{ij}\|\le \delta$.
Define $\mathcal T_{\widehat\zeta}$ by replacing $\zeta$ with $\widehat\zeta$
in the formula for $\mathcal T_\zeta$. Then, for every $f\in B_M$,
$\|\mathcal T_{\widehat\zeta}f-\mathcal T_\zeta f\|_2
\le
C_{\mathrm{chart}}\delta\|f\|_2$,
where
$C_{\mathrm{chart}}
=
\sum_{r=1}^{R_\xi}(L_{a_r}B_r+A_rL_{b_r})
+
\sum_{s=1}^{R_\eta}(L_{c_s}D_s+C_sL_{d_s})
+
L_\ell$.
In particular,
$\sup_{f\in B_M}\|\mathcal T_{\widehat\zeta}f-\mathcal T_\zeta f\|_2
\le
C_{\mathrm{chart}}M\delta$.
\end{lemma}
In particular, if the learned chart is $\delta$-close to the ideal chart, the induced operator error grows at most linearly with $\delta$. This guarantees stability with respect to chart perturbations, ensuring that small errors in the learned chart do not significantly degrade the resulting operator.
\begin{theorem}[Approximation of charted axial low-rank operators by one-block CATO]
\label{thm:favorable-chart}
Let $\widetilde{\mathcal G}_\Phi:B_M\to\mathbb{R}^{H\times W}$
be $(R_\xi,R_\eta,\varepsilon_{\mathrm{rk}})$-charted axial low-rank as defines in Definition \ref{def}. Then for every $\varepsilon_{\mathrm{nn}}>0$, there
exists a hidden width $C$ and parameters of a one-block core CATO with $R_\xi$
row heads and $R_\eta$ column heads such that
$\sup_{f\in B_M}
\|\mathcal N_\Theta(f,X)-\widetilde{\mathcal G}_\Phi f\|_2
\le
\varepsilon_{\mathrm{rk}}M+\varepsilon_{\mathrm{nn}}$.
Moreover, if the hypotheses of Lemma~\ref{lem:chart-stability} hold and
$\max_{i,j}\|\widehat\zeta_{ij}-\zeta_{ij}\|\le \delta$,
then one can choose a one-block core CATO of the same axial size such that
$\sup_{f\in B_M}
\|\mathcal N_\Theta(f,X)-\widetilde{\mathcal G}_\Phi f\|_2
\le
\varepsilon_{\mathrm{rk}}M
+
C_{\mathrm{chart}}M\delta
+
\varepsilon_{\mathrm{nn}}$.
\end{theorem}
In particular, we show that learning a coordinate chart can transform a complex operator into one that is effectively low-rank and therefore efficiently approximable. The proofs of Lemma \ref{lem:axial-realization}, Lemma \ref{lem:chart-stability}, and Theorem \ref{thm:favorable-chart} are provided in the Appendix \ref{proof}. Together, these results show that an appropriate coordinate system can simplify the target operator. When the operator has a simpler structure in chart space, CATO can represent it effectively; moreover, if the learned chart is sufficiently close to the ideal chart, the additional error remains small. Thus, chart learning is beneficial because it reduces the effective complexity of the operator class encountered by the network. This provides a theoretical explanation for why learning a chart can improve neural operator learning: when the chart renders the operator approximately axial and low-complexity, CATO achieves a small approximation error.

\section{Experiment Results}
\label{result}

\label{benchmark}
\paragraph{Benchmarks, baselines and implementation details.} We cover a wide range of different representative datasets, including Darcy and Navier-Stokes \cite{li2020fourier}, in the regular grid setting. In addition, we compared the method across irregular geometries, including Airfoil, Plasticity, and Pipe \cite{li2023fourier}, all defined on structured meshes, and Elasticity \cite{li2023fourier}, represented as point clouds. More details can be found in the appendix \ref{Benchmark}.

We compared CATO against 15 baselines that covered a wide range of neural operators, including frequency-based and transformer-based. For the frequency-based model, we compared FNO \cite{li2020fourier}, U-FNO \cite{wen2022u}, WMT \cite{gupta2021multiwavelet}, F-FNO \cite{tran2021factorized}, U-NO \cite{rahman2022u}. GEO-FNO \cite{li2023fourier}  and LSM \cite{wu2023solving}. For the transformer-based method,  we compared with Galerkin \cite{cao2021choose}, HT-NET \cite{liu2022ht}, OFormer \cite{li2022transformer}, GNOT \cite{hao2023gnot}, FactFormer \cite{li2023scalable}, ONO \cite{xiao2023improved}, Transolver \cite{wu2024transolver} and SAOT\cite{zhou2026saot}. For a fair comparison with Transolver, we set both the number of attention heads and the number of layers to 8. For all methods, we conduct all experiments on a single NVIDIA A100 40GB GPU.

\paragraph{Architecture by geometry type.} CATO is a geometry-first framework built around a learned chart, with the attention operator instantiated according to the input topology. For regular grid or structured mesh, the inputs have a regular-grid or structured-mesh layout, so we use the charted axial CATO block from Section \ref{CATO block}. For Elasticity, the input is an unordered point cloud with 972 nodes, where no canonical row-column factorization exists. We therefore use CATO-PC, a geometry-aware point-cloud variant. CATO-PC keeps the same learned chart as the core geometric representation, but replaces axial row/column attention with chart-conditioned physical attention for global operator modeling and a KNN-based local operator for neighborhood-level interactions. This is a deliberate topology-aware instantiation rather than a change in the central idea: across all datasets, CATO first learns a geometry-adaptive chart, and only the attention pattern is adapted to the data structure. This makes the point-cloud experiment a strength, as it demonstrates that chart learning generalizes beyond the axial-attention architecture. Additional details are provided in the Appendix \ref{catopc}.

\paragraph{Main results.} 
Table~\ref{tab:mainres_standard} presents a comprehensive comparison of CATO with standard and recent neural operators on six representative benchmarks covering point clouds, structured meshes, and regular grids. Across all datasets, CATO attains the lowest relative $L^2$ error, demonstrating consistent superiority over both frequency-domain methods and attention-based architectures. Notably, although recent approaches such as Transolver and SAOT already provide strong performance, CATO further improves upon these competitive baselines and exhibits the most balanced accuracy across heterogeneous discretizations, geometries, and physical regimes. On average, CATO reduces the relative error by approximately $27\%$ compared with the strongest competing method. The gains are particularly pronounced on challenging fluid and nonlinear material benchmarks, including Navier--Stokes, where the error decreases from $0.0675$ to $0.0319$ ($52.7\%$ reduction), and Plasticity, where the error decreases from $0.0009$ to $0.0005$ ($44.4\%$ reduction). CATO also yields consistent improvements on Elasticity ($0.0081 \rightarrow 0.0070$), Airfoil ($0.0049 \rightarrow 0.0041$), Pipe ($0.0050 \rightarrow 0.0038$), and Darcy ($0.0049 \rightarrow 0.0042$). These results suggest that CATO is not specialized to a particular discretization type or PDE family, but instead provides a robust and broadly applicable operator-learning framework. Overall, the superior and stable performance across both solid- and fluid-mechanics benchmarks highlights the effectiveness of CATO in learning accurate surrogate solution operators for diverse scientific computing problems.
\begin{table*}[t]
    \centering
    \caption{
        Experimental results are compared across different methods and PDE types. The results are reported as relative $L^2$ errors. \colorbox[HTML]{BBFFBB}{Green} indicates the best result, while \underline{underlining} indicates the second-best result. (*) indicates that the result was reproduced by us.
    }
    \label{tab:mainres_standard}
    \vspace{2pt}
    \small
    \setlength{\tabcolsep}{8pt}
    \renewcommand{\arraystretch}{1.12}
    \begin{tabular}{l|ccc|cc|c}
        \toprule
        \multirow{2}{*}{\textbf{Model}} 
        & \multicolumn{3}{c|}{\textbf{Structured Mesh}} 
        & \multicolumn{2}{c|}{\textbf{Regular Grid}}
        & \textbf{Point Cloud} \\
        \cmidrule(lr){2-4} \cmidrule(lr){5-6} \cmidrule(lr){7-7}
        & Plasticity & Airfoil & Pipe & NS & Darcy & Elasticity \\
        \midrule
        FNO (2021) \cite{li2020fourier}        & /      & /      & /      & 0.1556 & 0.0108 & /      \\
        WMT (2021) \cite{gupta2021multiwavelet}         & 0.0076 & 0.0075 & 0.0077 & 0.1541 & 0.0082 & 0.0359 \\
        U-FNO (2022) \cite{wen2022u}       & 0.0039 & 0.0269 & 0.0056 & 0.2231 & 0.0183 & 0.0239 \\
        GEO-FNO (2022) \cite{li2023fourier}     & 0.0074 & 0.0138 & 0.0067 & 0.1556 & 0.0108 & 0.0229 \\
        U-NO (2023) \cite{rahman2022u}       & 0.0034 & 0.0078 & 0.0100 & 0.1713 & 0.0113 & 0.0258 \\
        F-FNO (2023) \cite{tran2021factorized}      & 0.0047 & 0.0078 & 0.0070 & 0.2322 & 0.0077 & 0.0263 \\
        LSM (2023) \cite{wu2023solving}         & 0.0025 & 0.0059 & 0.0050 & 0.1535 & 0.0065 & 0.0218 \\
        \midrule
        Galerkin (2021) \cite{cao2021choose}    & 0.0120 & 0.0118 & 0.0098 & 0.1401 & 0.0084 & 0.0240 \\
        HT-Net (2022) \cite{liu2024mitigating}      & 0.0333 & 0.0065 & 0.0059 & 0.1847 & 0.0079 & /      \\
        OFormer (2023) \cite{li2022transformer}    & 0.0017 & 0.0183 & 0.0168 & 0.1705 & 0.0124 & 0.0183 \\
        GNOT (2023) \cite{hao2023gnot}      & 0.0336 & 0.0076 & \underline{0.0047} & 0.1380 & 0.0105 & 0.0086 \\
        FactFormer (2023) \cite{li2023scalable} & 0.0312 & 0.0071 & 0.0060 & 0.1214 & 0.0109 & /      \\
        ONO (2024) \cite{xiao2023improved}         & 0.0048 & 0.0061 & 0.0052 & 0.1195 & 0.0076 & 0.0118 \\
        Transolver* (2024) \cite{wu2024transolver}  & 0.0013 & 0.0053 & 0.0050 & 0.0920 & 0.0058 & \underline{0.0081} \\
        SAOT* (2026) \cite{zhou2026saot}       & \underline{0.0009} & \underline{0.0049} & 0.0061 & \underline{0.0675} & \underline{0.0049} & 0.0085 \\
        \midrule
        \cellcolor[HTML]{D9FFD9} \textbf{CATO (Ours)} 
                     & \cellcolor[HTML]{D9FFD9}\textbf{0.0005} 
                     & \cellcolor[HTML]{D9FFD9}\textbf{0.0041} 
                     & \cellcolor[HTML]{D9FFD9}\textbf{0.0038} 
                     & \cellcolor[HTML]{D9FFD9}\textbf{0.0319} 
                     & \cellcolor[HTML]{D9FFD9}\textbf{0.0042}
                     & \cellcolor[HTML]{D9FFD9}\textbf{0.0070} \\
        \textbf{\textbf{Error Reduction} ($\downarrow$)}
                     & 44.44\% & 16.33\% & 19.15\% & 52.74\% & 14.29\% & 13.58\% \\
        \bottomrule
    \end{tabular}
    \vspace{-4pt}
\end{table*}
\begin{figure}[t!]
    \centering
    \includegraphics[width=1\linewidth]{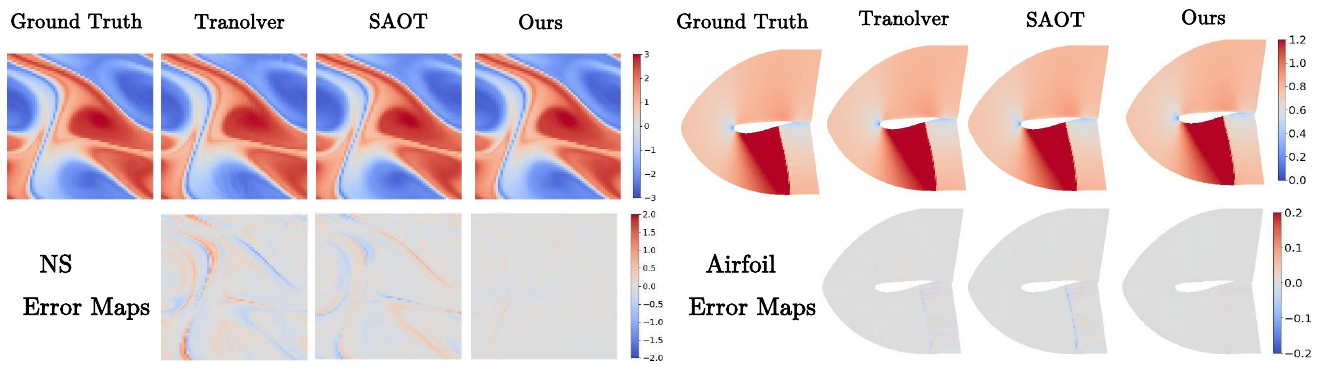}
    \caption{Visual comparison on Navier–Stokes and Airfoil benchmarks. Top: ground truth and predictions from Transolver, SAOT, and our method. Bottom: corresponding error maps}
    \label{fig:viua}
\end{figure}
Figure \ref{fig:viua} presents a qualitative comparison of prediction results on two challenging fluid-dynamics benchmarks: Navier-Stokes flow and Airfoil flow. CATO’s predictions are visually closer to the ground truth than both Transolver and SAOT, especially around turbulent vortices in Navier-Stokes and shock/wake regions near the airfoil, where its error maps are much lighter and more localized. This shows that CATO captures complex fluid dynamics and sharp physical transitions more accurately. More visualization results are available in the Appendix \ref{More abaltion}.

\paragraph{Scaling \& efficiency. } To further assess the scalability of CATO on the Darcy, we systematically evaluate its performance under variations in training sample size, spatial resolution, network depth, and embedding dimension. As shown in Figure~\ref{scale}, CATO consistently achieves lower relative $L^2$ error than SAOT across all data regimes, demonstrating superior data efficiency and robustness. Under resolution scaling, CATO maintains a clear performance advantage as the grid resolution increases and continues to benefit from finer discretizations, indicating strong generalization capability across spatial scales. In addition, CATO remains stable across changes in the number of layers and embedding dimensions, whereas SAOT consistently exhibits higher error under the same settings. These results demonstrate that CATO scales reliably across data, resolution, architecture depth, and feature dimension, highlighting its effectiveness as a robust and efficient neural operator for PDE.

To further analyze the computational efficiency of the proposed model, we present its efficiency metrics in Figure \ref{eff} (a) and (b) compared with Transolver and SAOT on the Darcy and Pipe benchmarks. Specifically, on the Darcy benchmark, our model achieves substantially lower computational cost, reducing the number of parameters by around 85\% and GFLOPs by 69\% compared to SAOT. On the Pipe benchmark, our model further demonstrates clear efficiency gains, achieving the lowest GFLOPs and shortest training time among all compared methods. In addition, the bubble size indicates that our model uses fewer parameters than both baselines, showing that it is more compact while remaining computationally efficient. These results highlight the favorable efficiency of our model in terms of training time, computational cost, and parameter count across different PDE benchmarks.

\begin{figure}[t!]
    \centering
    \includegraphics[width=1\linewidth]{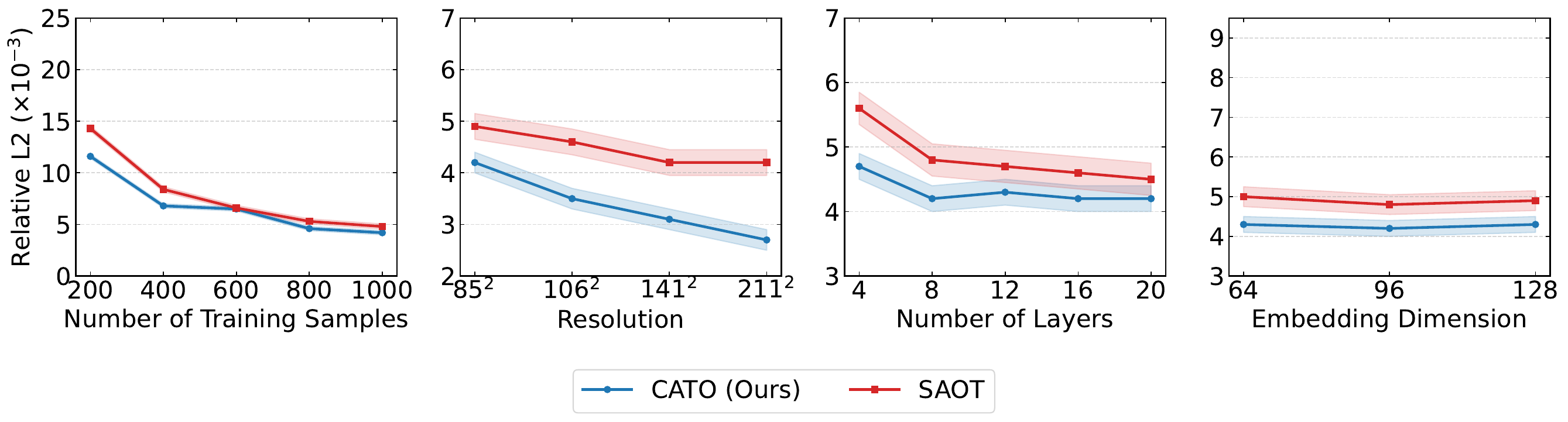}
    \caption{Model scaling performance on Darcy flow. We compare our method with SAOT across training sample size, resolution, layer count, and embedding dimension.
    }
    \label{scale}
\end{figure}

\begin{figure}[t!]
    \centering
    \includegraphics[width=1\linewidth]{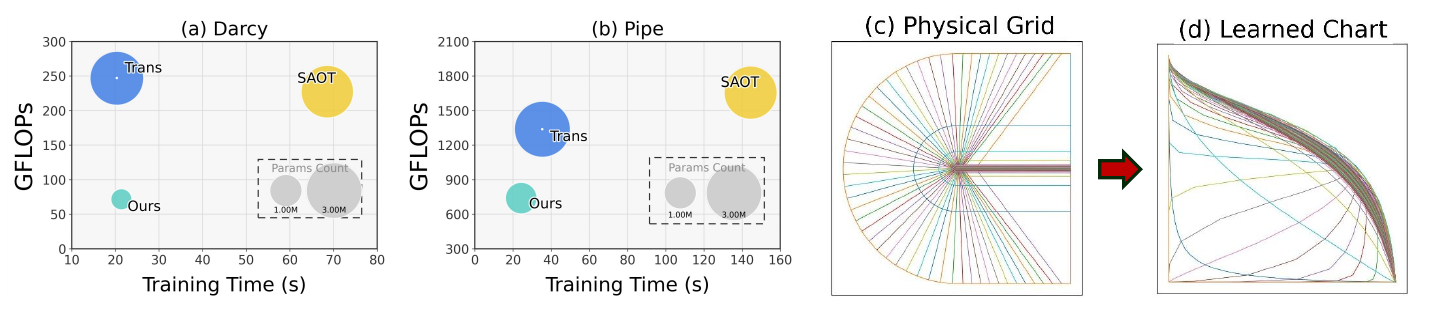}
    \caption{ (a) and (b) show the efficiency on Darcy and Pipe in terms of training time per epoch, number of parameters, and GFLOPs.  (c) and (d) show the physical grid and learned chart space
    }
    \label{eff}
\end{figure}

\paragraph{Model analysis.} 
Figure~\ref{eff}c--d illustrates the transformation from the physical grid to the learned chart space. The learned chart acts as a geometry-adaptive coordinate system that concentrates resolution along dynamically significant directions while flattening variations induced by the underlying physics. This transformation simplifies the operator representation, making it more structured and easier to approximate than in the original coordinate space.
To quantify this effect, we analyze the learned chart via principal component analysis. We observe that $94.0\%$ of the variance is captured by the first principal component, while the second accounts for only $6.0\%$. This strong anisotropy indicates that the learned chart collapses the original two-dimensional domain onto a nearly one-dimensional manifold, aligned with the dominant physical direction (e.g., pressure gradient in Darcy flow). The participation-ratio effective dimension of $1.126$ further confirms that the intrinsic dimensionality is significantly reduced.
This directly supports our theoretical hypothesis: the learned chart induces a low-dimensional, approximately separable structure in which the solution operator becomes easier to approximate, explaining why axial attention is particularly effective in the chart space.
Finally, we compare against a coordinate-normalization baseline that removes translation and scaling while preserving the original coordinate structure. Normalization yields an error of $0.0045$, whereas the learned chart achieves $0.0041$. This demonstrates that the gains arise from learning a geometry-adaptive representation, rather than simple rescaling, and validates that chart learning provides a complementary source of improvement beyond architectural design.

\section{Conclusion}
This paper presents CATO, a charted axial transformer operator for solving PDEs on general geometries. By learning a continuous geometry-adaptive chart, applying efficient axial attention in chart space, and incorporating local operators with mixed-form value and derivative supervision, CATO captures both long-range physical interactions and local differential structures. Experiments on six PDE benchmarks show that CATO consistently achieves state-of-the-art accuracy across regular grids, structured meshes, and point clouds, while theoretical analysis supports its ability to approximate low-complexity solution operators under a favorable chart. 
More broadly, CATO highlights the importance of learning coordinate representations for neural operator design. These results suggest that coordinate-aware attention may provide a scalable and physically meaningful framework for scientific machine learning.

\paragraph{Limitations.} While CATO demonstrates strong performance on 2D PDE benchmarks, extending the approach to large-scale 3D and multiphysics settings remains future work.

\section*{Acknowledgments}
CWC is supported by the  Swiss National Science Foundation (SNSF)  under grant number 20HW-1\_220785. It also acknowledge CMI, University of Cambridge. CBS acknowledges support from the Philip Leverhulme Prize, the Royal Society Wolfson Fellowship, the EPSRC advanced career fellowship EP/V029428/1, EPSRC grants EP/S026045/1 and EP/T003553/1, EP/N014588/1, EP/T017961/1, the Wellcome Innovator Awards 215733/Z/19/Z and 221633/Z/20/Z, CCMI and the Alan Turing Institute. 
AIAR gratefully acknowledges the support of the Yau Mathematical Sciences Center, Tsinghua University. This work is also supported by the Tsinghua University Dushi Program.

\bibliography{main.bib}
\bibliographystyle{plain}


\newpage
\appendix

\addtocontents{toc}{\protect\setcounter{tocdepth}{2}}

\textbf{\Large CATO: Charted Attention for Neural PDE Operators -- Appendix}
\label{sec:appendix}

\noindent\rule{\textwidth}{0.6pt}
\tableofcontents
\noindent\rule{\textwidth}{0.6pt}
\section{Table of Notation}
\label{app:notation}

\begin{longtable}{
  >{\raggedright\arraybackslash}p{0.23\linewidth}
  >{\raggedright\arraybackslash}p{0.52\linewidth}
  >{\raggedright\arraybackslash}p{0.19\linewidth}
}
\caption{Notation used throughout the paper.}
\label{tab:notation}
\\
\toprule
\textbf{Symbol} & \textbf{Description} & \textbf{Shape / Domain} \\
\midrule
\endfirsthead

\toprule
\textbf{Symbol} & \textbf{Description} & \textbf{Shape / Domain} \\
\midrule
\endhead

\midrule
\multicolumn{3}{r}{\emph{Continued on next page}} \\
\endfoot

\bottomrule
\endlastfoot

\multicolumn{3}{l}{\textbf{Mesh, inputs, and outputs}} \\
\addlinespace

\(\Omega\) 
& Physical domain on which the PDE is defined. 
& \(\Omega \subset \mathbb{R}^{2}\) \\

\(H, W\) 
& Number of mesh points along the two structured mesh directions. 
& Positive integers \\

\(N\) 
& Total number of spatial nodes. 
& \(N = HW\) \\

\(B\) 
& Batch size. 
& Positive integer \\

\((i,j)\) 
& Two-dimensional mesh index. 
& \(1 \leq i \leq H,\; 1 \leq j \leq W\) \\

\(n\) 
& Flattened node index. 
& \(1 \leq n \leq N\) \\

\(\mathbf{x}_{ij}\) 
& Physical coordinate of node \((i,j)\). We write \(\mathbf{x}_{ij}=(x_{ij},y_{ij})\) to distinguish the vector coordinate from its scalar components. 
& \(\mathbb{R}^{2}\) \\

\(X\) 
& Collection of all mesh coordinates. 
& \(\{\mathbf{x}_{ij}\}_{i,j}\), or \(\mathbb{R}^{B \times N \times 2}\) \\

\(\mathbf{f}_{ij}\) 
& Optional node-wise auxiliary input features, such as coefficients, source terms, or field descriptors. 
& \(\mathbb{R}^{d_f}\) \\

\(F\) 
& Collection of auxiliary input features over all mesh nodes. 
& \(\mathbb{R}^{B \times N \times d_f}\) \\

\(d_f\) 
& Dimension of the auxiliary input feature vector. 
& Nonnegative integer \\

\(z^{\mathrm{in}}_{ij}\) 
& Input token at node \((i,j)\), formed by concatenating coordinate and auxiliary features. 
& \([\mathbf{x}_{ij},\mathbf{f}_{ij}]\) if \(d_f>0\); otherwise \(\mathbf{x}_{ij}\) \\

\(u\) 
& Ground-truth scalar solution field. 
& \(\mathbb{R}^{H \times W}\) or \(\mathbb{R}^{B \times N \times 1}\) \\

\(\hat{u}\) 
& Predicted scalar solution field. 
& \(\mathbb{R}^{H \times W}\) or \(\mathbb{R}^{B \times N \times 1}\) \\

\(\hat{\mathbf{q}}\) 
& Auxiliary vector output used as a gradient-like flux proxy. 
& \(\mathbb{R}^{B \times N \times 2}\) \\

\(\nabla u\) 
& Spatial gradient of the target scalar field. 
& \(\mathbb{R}^{B \times N \times 2}\) \\

\(\nabla \hat{u}\) 
& Reconstructed spatial gradient of the predicted scalar field. 
& \(\mathbb{R}^{B \times N \times 2}\) \\

\(\mathcal{G}_{\theta}\) 
& Learned neural solution operator mapping mesh coordinates and optional features to the solution field. 
& \((X,F) \mapsto \hat{u}\) \\

\addlinespace
\multicolumn{3}{l}{\textbf{CATO architecture}} \\
\addlinespace

\(\Phi_{\mathrm{pre}}\) 
& Input lifting network that maps each input token to the latent feature space. 
& MLP \\

\(C\) 
& Latent embedding dimension. 
& Positive integer \\

\(H^{(\ell)}\) 
& Hidden representation after the \(\ell\)-th CATO block. 
& \(\mathbb{R}^{B \times H \times W \times C}\) \\

\(\mathbf{h}^{(\ell)}_{ij}\) 
& Hidden feature vector at node \((i,j)\) after layer \(\ell\). 
& \(\mathbb{R}^{C}\) \\

\(L\) 
& Number of stacked CATO blocks. 
& Positive integer \\

\(\mathrm{LN}(\cdot)\) 
& Layer normalization. 
& -- \\

\(\mathrm{MLP}(\cdot)\) 
& Pointwise feed-forward network used inside each block. 
& -- \\

\(\mathrm{DWConv}\) 
& Depthwise convolution used in the local operator branch. 
& \(k \times k\) convolution \\

\(\mathrm{PWConv}\) 
& Pointwise convolution used in the local operator branch. 
& \(1 \times 1\) convolution \\

\(\mathbf{w}_{u}, b_u\) 
& Linear readout parameters for the scalar prediction head. 
& \(\mathbf{w}_{u}\in\mathbb{R}^{C}\) \\

\(W_q, \mathbf{b}_q\) 
& Linear readout parameters for the auxiliary flux head. 
& \(W_q\in\mathbb{R}^{2\times C}\) \\

\addlinespace
\multicolumn{3}{l}{\textbf{Learned chart and positional encoding}} \\
\addlinespace

\(\Phi_{\mathrm{chart}}\) 
& Learned continuous chart mapping physical coordinates to latent chart coordinates. 
& \(\mathbb{R}^{2} \to [-1,1]^{2}\) \\

\(\boldsymbol{\zeta}_{ij}\) 
& Learned chart coordinate of node \((i,j)\). 
& \((\xi_{ij},\eta_{ij})\in[-1,1]^{2}\) \\

\(\xi_{ij}\) 
& First chart coordinate, used for row-wise axial attention. 
& \([-1,1]\) \\

\(\eta_{ij}\) 
& Second chart coordinate, used for column-wise axial attention. 
& \([-1,1]\) \\

\(K\) 
& Compact chart domain containing all learned chart coordinates. 
& \(K\subset[-1,1]^2\) \\

\(V_1,V_2,c_1,c_2\) 
& Parameters of the chart MLP. 
& -- \\

\(\theta\) 
& RoPE base parameter. 
& \(\theta>0\) \\

\(\omega_r\) 
& Angular frequency for the \(r\)-th RoPE channel pair. 
& \(\omega_r=\theta^{-2r/d_h}\) \\

\(R(p)\) 
& Continuous rotary positional encoding matrix evaluated at position \(p\). 
& Block-diagonal rotation matrix \\

\(p\) 
& Continuous positional input to RoPE; in CATO this is a chart coordinate. 
& \(p=\xi_{ij}\) or \(p=\eta_{ij}\) \\

\addlinespace
\multicolumn{3}{l}{\textbf{Charted axial attention}} \\
\addlinespace

\(M\) 
& Number of attention heads. 
& Positive integer \\

\(d_h\) 
& Per-head dimension. 
& \(d_h=C/M\) \\

\(W_Q,W_K,W_V\) 
& Query, key, and value projection matrices. 
& -- \\

\(\mathbf{q}_{ij},\mathbf{k}_{ij},\mathbf{v}_{ij}\) 
& Query, key, and value vectors at node \((i,j)\). 
& \(\mathbb{R}^{C}\) before head splitting \\

\(\mathbf{q}^{(m)}_{ij},\mathbf{k}^{(m)}_{ij},\mathbf{v}^{(m)}_{ij}\) 
& Query, key, and value vectors for attention head \(m\). 
& \(\mathbb{R}^{d_h}\) \\

\(\tilde{\mathbf{q}}^{(m)}_{ij},\tilde{\mathbf{k}}^{(m)}_{ij}\) 
& RoPE-rotated query and key vectors. 
& \(\mathbb{R}^{d_h}\) \\

\(\alpha^{(m)}_{i,j,t}\) 
& Row-attention weight from node \((i,j)\) to node \((i,t)\) in head \(m\). 
& Softmax-normalized \\

\(\beta^{(m)}_{i,j,s}\) 
& Column-attention weight from node \((i,j)\) to node \((s,j)\) in head \(m\). 
& Softmax-normalized \\

\(\mathrm{Attn}_{\mathrm{row}}(h;\xi)\) 
& Row-wise axial attention using the chart coordinate \(\xi\). 
& -- \\

\(\mathrm{Attn}_{\mathrm{col}}(h;\eta)\) 
& Column-wise axial attention using the chart coordinate \(\eta\). 
& -- \\

\(\mathcal{A}(h,\boldsymbol{\zeta})\) 
& Charted axial attention output, defined as the sum of row and column attention. 
& \(\mathrm{Attn}_{\mathrm{row}}(h;\xi)+\mathrm{Attn}_{\mathrm{col}}(h;\eta)\) \\

\(W^{\mathrm{row}}_O,W^{\mathrm{col}}_O\) 
& Output projections for row and column attention. 
& -- \\

\addlinespace
\multicolumn{3}{l}{\textbf{Physical loss and discrete gradients}} \\
\addlinespace

\(\Delta_i u_{ij}\) 
& Centered finite difference of \(u\) along the first mesh direction. 
& \(u_{i+1,j}-u_{i-1,j}\) \\

\(\Delta_j u_{ij}\) 
& Centered finite difference of \(u\) along the second mesh direction. 
& \(u_{i,j+1}-u_{i,j-1}\) \\

\(\Delta_i \mathbf{x}_{ij}\) 
& Centered coordinate difference along the first mesh direction. 
& \(\mathbf{x}_{i+1,j}-\mathbf{x}_{i-1,j}\) \\

\(\Delta_j \mathbf{x}_{ij}\) 
& Centered coordinate difference along the second mesh direction. 
& \(\mathbf{x}_{i,j+1}-\mathbf{x}_{i,j-1}\) \\

\(a,b,c,d\) 
& Components of the local coordinate-difference vectors, with \(\Delta_i\mathbf{x}_{ij}=(a,b)\) and \(\Delta_j\mathbf{x}_{ij}=(c,d)\). 
& Scalars \\

\(u_x,u_y\) 
& Reconstructed physical gradient components at node \((i,j)\). 
& Scalars \\

\(ad-bc\) 
& Determinant of the local coordinate-difference matrix. Nonzero determinant ensures a locally nonsingular gradient reconstruction. 
& Scalar \\

\(\mathrm{Grad}(u,X)\) 
& Mesh-consistent gradient reconstruction operator applied to scalar field \(u\) on mesh \(X\). 
& \(\mathbb{R}^{B\times N\times 2}\) \\

\(\mathcal{L}_{\mathrm{val}}\) 
& Relative \(L^2\) value loss between \(\hat{u}\) and \(u\). 
& Scalar \\

\(\mathcal{L}_{\mathrm{grad}}\) 
& Gradient-matching loss between \(\nabla\hat{u}\) and \(\nabla u\). 
& Scalar \\

\(\mathcal{L}_{\mathrm{flux}}\) 
& Auxiliary flux loss between \(\hat{\mathbf{q}}\) and \(\nabla u\). 
& Scalar \\

\(\mathcal{L}_{\mathrm{cons}}\) 
& Consistency loss between \(\hat{\mathbf{q}}\) and \(\nabla\hat{u}\). 
& Scalar \\

\(\lambda_g,\lambda_f,\lambda_c\) 
& Weights for the gradient, flux, and consistency losses. 
& Nonnegative scalars \\

\(\mathcal{L}\) 
& Total training loss. 
& \(\mathcal{L}_{\mathrm{val}}+\lambda_g\mathcal{L}_{\mathrm{grad}}+\lambda_f\mathcal{L}_{\mathrm{flux}}+\lambda_c\mathcal{L}_{\mathrm{cons}}\) \\

\(\varepsilon\) 
& Small numerical constant used for stable relative-error computation. 
& Positive scalar \\

\addlinespace
\multicolumn{3}{l}{\textbf{Theory}} \\
\addlinespace

\(f\) 
& Scalar input field used in the theoretical analysis. 
& \(\mathbb{R}^{H\times W}\) \\

\(B_M\) 
& \(L^2\)-bounded input ball used in the approximation analysis. 
& \(\{f\in\mathbb{R}^{H\times W}:\|f\|_2\leq M\}\) \\

\(M\) 
& Radius of the input ball \(B_M\). 
& Positive scalar \\

\(\widetilde{\mathcal{G}}_{\Phi}\) 
& Target operator expressed with respect to a chart. 
& \(B_M\to\mathbb{R}^{H\times W}\) \\

\(T_{\boldsymbol{\zeta}}\) 
& Finite-rank charted axial operator associated with chart \(\boldsymbol{\zeta}\). 
& \(B_M\to\mathbb{R}^{H\times W}\) \\

\(\mathcal{R}\) 
& Residual operator in the charted axial low-rank decomposition. 
& \(\widetilde{\mathcal{G}}_{\Phi}=T_{\boldsymbol{\zeta}}+\mathcal{R}\) \\

\(R_{\xi},R_{\eta}\) 
& Row-wise and column-wise axial ranks; equivalently, the number of row and column components in the theoretical decomposition. 
& Positive integers \\

\(a_r,b_r\) 
& Continuous coefficient functions used in the row-wise part of \(T_{\boldsymbol{\zeta}}\). 
& \(K\to\mathbb{R}\) \\

\(c_s,d_s\) 
& Continuous coefficient functions used in the column-wise part of \(T_{\boldsymbol{\zeta}}\). 
& \(K\to\mathbb{R}\) \\

\(\ell\) 
& Continuous coefficient function for the local pointwise term in \(T_{\boldsymbol{\zeta}}\). 
& \(K\to\mathbb{R}\) \\

\(m_r(i;f)\) 
& Row-wise averaged feature in the theoretical construction. 
& \(\frac{1}{W}\sum_{t=1}^{W} b_r(\boldsymbol{\zeta}_{it})f_{it}\) \\

\(n_s(j;f)\) 
& Column-wise averaged feature in the theoretical construction. 
& \(\frac{1}{H}\sum_{p=1}^{H} d_s(\boldsymbol{\zeta}_{pj})f_{pj}\) \\

\(\varepsilon_{\mathrm{rk}}\) 
& Error of the charted axial low-rank approximation. 
& Nonnegative scalar \\

\(\varepsilon_{\mathrm{nn}}\) 
& Neural approximation error of the one-block CATO realization. 
& Positive scalar \\

\(N_{\Theta}\) 
& Neural operator realized by a one-block CATO core followed by a linear readout. 
& \(B_M\to\mathbb{R}^{H\times W}\) \\

\(\hat{\boldsymbol{\zeta}}_{ij}\) 
& Perturbed or learned approximation of the ideal chart coordinate. 
& \(K\subset[-1,1]^2\) \\

\(\delta\) 
& Maximum chart perturbation size. 
& \(\max_{i,j}\|\hat{\boldsymbol{\zeta}}_{ij}-\boldsymbol{\zeta}_{ij}\|\) \\

\(A_r,B_r,C_s,D_s,L_0\) 
& Uniform bounds on the coefficient functions \(a_r,b_r,c_s,d_s,\ell\), respectively. 
& Nonnegative scalars \\

\(L_{a_r},L_{b_r},L_{c_s},L_{d_s},L_{\ell}\) 
& Lipschitz constants of the coefficient functions \(a_r,b_r,c_s,d_s,\ell\), respectively. 
& Nonnegative scalars \\

\(C_{\mathrm{chart}}\) 
& Stability constant controlling the effect of chart perturbations on the axial operator. 
& \(\sum_{r=1}^{R_{\xi}}(L_{a_r}B_r+A_rL_{b_r})+\sum_{s=1}^{R_{\eta}}(L_{c_s}D_s+C_sL_{d_s})+L_{\ell}\) \\

\end{longtable}
\section{Further Theoretical Results}
\label{proof}
In this section, we will provide the complete proof of Lemma 1, Lemma 2, and Theorem 1.

\begin{lemma}[Neural realization of charted axial finite-rank operators]
\label{applemma}
Let $\mathcal T_\zeta:B_M\to\mathbb{R}^{H\times W}$ be given by
\begin{equation}
(\mathcal T_\zeta f)_{ij}
=
\sum_{r=1}^{R_\xi}
a_r(\zeta_{ij})
\Big(\frac1W\sum_{t=1}^W b_r(\zeta_{it})f_{it}\Big)
+
\sum_{s=1}^{R_\eta}
c_s(\zeta_{ij})
\Big(\frac1H\sum_{p=1}^H d_s(\zeta_{pj})f_{pj}\Big)
+
\ell(\zeta_{ij})f_{ij},
\end{equation}
where $a_r,b_r,c_s,d_s,\ell$ are continuous on $K$.
Then for every $\varepsilon_{\mathrm{nn}}>0$, there exist a hidden width $C$
and parameters of a one-block core CATO with $R_\xi$ row heads and $R_\eta$
column heads such that
\begin{equation}
\sup_{f\in B_M}\|\mathcal N_\Theta(f,X)-\mathcal T_\zeta f\|_2
\le
\varepsilon_{\mathrm{nn}}.
\end{equation}
\end{lemma}
\begin{proof}
We first define the row part and column part of the operator:
\begin{equation}
m_r(i;f):=\frac1W\sum_{t=1}^W b_r(\zeta_{it})f_{it},
\qquad
n_s(j;f):=\frac1H\sum_{p=1}^H d_s(\zeta_{pj})f_{pj},
\end{equation}
Then we can write the output of $\mathcal T_\zeta$ as follows:
\begin{equation}
(\mathcal T_\zeta f)_{ij}
=
\sum_{r=1}^{R_\xi} a_r(\zeta_{ij})\,m_r(i;f)
+
\sum_{s=1}^{R_\eta} c_s(\zeta_{ij})\,n_s(j;f)
+
\ell(\zeta_{ij})f_{ij}.
\end{equation}

We can choose any compact set $\mathcal X\subset\mathbb{R}^2$ that containing all mesh $\{x_{ij}\}$. Since $\Phi_{\mathrm{chart}}$ and
$a_r,b_r,c_s,d_s,\ell$ are continuous, and we know that the composition of two functions is also continuous, then we know that the following functions on
$\mathcal X\times[-M,M]$ are also continuous:
\[
g^{(P)}_r(x,z):=a_r(\Phi_{\mathrm{chart}}(x)),
\qquad
g^{(U)}_r(x,z):=b_r(\Phi_{\mathrm{chart}}(x))\,z,
\]
\begin{equation}
g^{(Q)}_s(x,z):=c_s(\Phi_{\mathrm{chart}}(x)),
\qquad
g^{(V)}_s(x,z):=d_s(\Phi_{\mathrm{chart}}(x))\,z,
\end{equation}
\[
g^{(\Lambda)}(x,z):=\ell(\Phi_{\mathrm{chart}}(x)),
\qquad
g^{(Z)}(x,z):=z.
\]

Furthermore, let
\begin{equation}
A_r:=\|a_r\|_\infty,\quad
B_r:=\|b_r\|_\infty,\quad
C_s:=\|c_s\|_\infty,\quad
D_s:=\|d_s\|_\infty,\quad
L_0:=\|\ell\|_\infty.
\end{equation}
Define the compact set
\begin{equation}
\begin{aligned}
\mathcal D
:={}&
\prod_{r=1}^{R_\xi}[-A_r-1,A_r+1]
\times
\prod_{r=1}^{R_\xi}[-B_rM-1,B_rM+1]
\times
\prod_{s=1}^{R_\eta}[-C_s-1,C_s+1] \\
&\times
\prod_{s=1}^{R_\eta}[-D_sM-1,D_sM+1]
\times
[-L_0-1,L_0+1]
\times
[-M-1,M+1].
\end{aligned}
\end{equation}
On compact set $\mathcal D$, we define the following function:
\begin{equation}
F\big((p_r)_{r=1}^{R_\xi},(u_r)_{r=1}^{R_\xi},
      (q_s)_{s=1}^{R_\eta},(v_s)_{s=1}^{R_\eta},\lambda,z\big)
:=
\sum_{r=1}^{R_\xi} p_r u_r
+
\sum_{s=1}^{R_\eta} q_s v_s
+
\lambda z.
\end{equation}
Since $F$ is continuous on the compact set $\mathcal D$, it is uniformly
continuous. Hence there exists $\tau\in(0,1)$ such that whenever
$y,\widetilde y\in\mathcal D$ satisfy
\begin{equation}
\|y-\widetilde y\|_\infty\le \tau,
\end{equation}
we have
\begin{equation}
|F(y)-F(\widetilde y)|\le \frac{\varepsilon_{\mathrm{nn}}}{2\sqrt N}.
\end{equation}
Assume that there exists a hidden width $C$ large enough so that 
distinct scalar channels can be reserved for
\begin{equation}
\{P_r,U_r,M_r\}_{r=1}^{R_\xi},
\qquad
\{Q_s,V_s,N_s\}_{s=1}^{R_\eta},
\qquad
\Lambda,\ Z,\ O.
\end{equation}
By universal approximation for pointwise MLPs on compact sets, choose
$\Phi_{\mathrm{pre}}$ so that for every $(i,j)$ and every $f\in B_M$,
the designated channels of $H^{(0)}_{ij}=\Phi_{\mathrm{pre}}(x_{ij},f_{ij})$
satisfy
\begin{equation}
|P_{r,ij}-a_r(\zeta_{ij})|\le \tau,
\qquad
|U_{r,ij}-b_r(\zeta_{ij})f_{ij}|\le \tau,
\end{equation}
\begin{equation}
|Q_{s,ij}-c_s(\zeta_{ij})|\le \tau,
\qquad
|V_{s,ij}-d_s(\zeta_{ij})f_{ij}|\le \tau,
\end{equation}
\begin{equation}
|\Lambda_{ij}-\ell(\zeta_{ij})|\le \tau,
\qquad
|Z_{ij}-f_{ij}|\le \tau,
\end{equation}
while the summary and output channels are initialized exactly to zero:
\begin{equation}
M_{r,ij}=0,\qquad N_{s,ij}=0,\qquad O_{ij}=0.
\end{equation}

We next construct the axial attention block. For each of the $R_\xi$ row heads,
set the query and key projections to zero. After continuous RoPE, the rotated
queries and keys remain zero, so all row-attention logits are zero and the
softmax weights are uniform:
\begin{equation}
\alpha^{(r)}_{i,j,t}=\frac1W
\qquad
\text{for all }i,j,t.
\end{equation}
Choose the value projection of row head $r$ to select the designated scalar
channel $U_r$ and set all other value coordinates of that head to zero. Then the
scalar output of row head $r$ at node $(i,j)$ is
\begin{equation}
\widehat m_r(i;f)
=
\frac1W\sum_{t=1}^W U_{r,it}.
\end{equation}
Hence
\begin{equation}
\big|\widehat m_r(i;f)-m_r(i;f)\big|
=
\left|
\frac1W\sum_{t=1}^W
\Big(U_{r,it}-b_r(\zeta_{it})f_{it}\Big)
\right|
\le
\frac1W\sum_{t=1}^W \tau
=
\tau.
\end{equation}
Choose the row output projection so that the output of row head $r$ is written
into the reserved summary channel $M_r$ and all other row-output channels are
zero.

Similarly, for each of the $R_\eta$ column heads, set the query and key
projections to zero, so that the column-attention weights are uniform:
\begin{equation}
\beta^{(s)}_{i,j,p}=\frac1H
\qquad
\text{for all }i,j,p.
\end{equation}
Choose the value projection of column head $s$ to select channel $V_s$, and let
the column output projection write the result into the reserved summary channel
$N_s$. Then the scalar output of column head $s$ at node $(i,j)$ is
\begin{equation}
\widehat n_s(j;f)
=
\frac1H\sum_{p=1}^H V_{s,pj},
\end{equation}
and therefore
\begin{equation}
\big|\widehat n_s(j;f)-n_s(j;f)\big|
=
\left|
\frac1H\sum_{p=1}^H
\Big(V_{s,pj}-d_s(\zeta_{pj})f_{pj}\Big)
\right|
\le
\tau.
\end{equation}

By construction, the axial block writes only into the summary channels
$M_r,N_s$. Therefore, after the residual update
\begin{equation}
\widetilde H = H^{(0)} + A(H^{(0)},\zeta),
\end{equation}
the channels $P_r,Q_s,\Lambda,Z$ remain unchanged, the output channel $O$
remains zero, and the summary channels satisfy
\begin{equation}
\widetilde M_{r,ij}=\widehat m_r(i;f),
\qquad
\widetilde N_{s,ij}=\widehat n_s(j;f).
\end{equation}

For each node $(i,j)$, define the exact tuple
\begin{equation}
y_{ij}(f)
:=
\Big(
(a_r(\zeta_{ij}))_{r=1}^{R_\xi},
(m_r(i;f))_{r=1}^{R_\xi},
(c_s(\zeta_{ij}))_{s=1}^{R_\eta},
(n_s(j;f))_{s=1}^{R_\eta},
\ell(\zeta_{ij}),
f_{ij}
\Big),
\end{equation}
and the approximate tuple
\begin{equation}
\widehat y_{ij}(f)
:=
\Big(
(\widetilde P_{r,ij})_{r=1}^{R_\xi},
(\widetilde M_{r,ij})_{r=1}^{R_\xi},
(\widetilde Q_{s,ij})_{s=1}^{R_\eta},
(\widetilde N_{s,ij})_{s=1}^{R_\eta},
\widetilde\Lambda_{ij},
\widetilde Z_{ij}
\Big).
\end{equation}
From the construction above, every component differs by at most $\tau$, hence
\begin{equation}
\|\widehat y_{ij}(f)-y_{ij}(f)\|_\infty\le \tau.
\end{equation}
Therefore, by the choice of $\tau$,
\begin{equation}
\left|F(\widehat y_{ij}(f)) - F(y_{ij}(f))\right|
\le
\frac{\varepsilon_{\mathrm{nn}}}{2\sqrt N}.
\end{equation}
Since
\begin{equation}
F(y_{ij}(f))=(\mathcal T_\zeta f)_{ij},
\end{equation}
it remains to approximate $F$ pointwise from the channels of $\widetilde H$.

By universal approximation on compact sets, choose the pointwise block MLP so
that its $O$-channel output satisfies
\begin{equation}
\left| \Psi_O(\widetilde H_{ij}) - F(\widehat y_{ij}(f)) \right|
\le
\frac{\varepsilon_{\mathrm{nn}}}{2\sqrt N}
\end{equation}
uniformly over all admissible $\widetilde H_{ij}$, while all other MLP output
channels are identically zero. Since the $O$-channel of $\widetilde H$ is zero,
the residual update
\begin{equation}
H^{(1)}=\widetilde H+\mathrm{MLP}(\widetilde H)
\end{equation}
yields
\begin{equation}
|H^{(1)}_{ij,O}-(\mathcal T_\zeta f)_{ij}|
\le
\frac{\varepsilon_{\mathrm{nn}}}{\sqrt N}.
\end{equation}
Finally, choose the readout to select the $O$-channel:
\begin{equation}
w_{\mathrm{out}}=e_O,\qquad b_{\mathrm{out}}=0.
\end{equation}
Then, for every $f\in B_M$,
\begin{equation}
\|\mathcal N_\Theta(f,X)-\mathcal T_\zeta f\|_2^2
=
\sum_{i=1}^H\sum_{j=1}^W
|H^{(1)}_{ij,O}-(\mathcal T_\zeta f)_{ij}|^2
\le
N\cdot \frac{\varepsilon_{\mathrm{nn}}^2}{N}
=
\varepsilon_{\mathrm{nn}}^2.
\end{equation}
Thus
\begin{equation}
\sup_{f\in B_M}\|\mathcal N_\Theta(f,X)-\mathcal T_\zeta f\|_2
\le
\varepsilon_{\mathrm{nn}}.
\end{equation}
\end{proof}

\begin{lemma}[Lipschitz stability with respect to chart perturbations]
\label{applem2}
Let $\mathcal T_\zeta$ be as in Lemma~\ref{applemma}, and assume
in addition that the coefficient functions are bounded and Lipschitz:
\begin{equation}
\|a_r\|_\infty\le A_r,\quad
\|b_r\|_\infty\le B_r,\quad
\|c_s\|_\infty\le C_s,\quad
\|d_s\|_\infty\le D_s,\quad
\|\ell\|_\infty\le L_0,
\end{equation}
and
\begin{equation}
\operatorname{Lip}(a_r)\le L_{a_r},\quad
\operatorname{Lip}(b_r)\le L_{b_r},\quad
\operatorname{Lip}(c_s)\le L_{c_s},\quad
\operatorname{Lip}(d_s)\le L_{d_s},\quad
\operatorname{Lip}(\ell)\le L_\ell.
\end{equation}
Let another chart $\widehat\zeta_{ij}\in K$ satisfy
\begin{equation}
\max_{i,j}\|\widehat\zeta_{ij}-\zeta_{ij}\|\le \delta.
\end{equation}
Define $\mathcal T_{\widehat\zeta}$ by replacing $\zeta$ with $\widehat\zeta$
in the formula for $\mathcal T_\zeta$. Then, for every $f\in B_M$,
\begin{equation}
\|\mathcal T_{\widehat\zeta}f-\mathcal T_\zeta f\|_2
\le
C_{\mathrm{chart}}\delta\|f\|_2,
\end{equation}
where
\begin{equation}
C_{\mathrm{chart}}
=
\sum_{r=1}^{R_\xi}(L_{a_r}B_r+A_rL_{b_r})
+
\sum_{s=1}^{R_\eta}(L_{c_s}D_s+C_sL_{d_s})
+
L_\ell.
\end{equation}
In particular,
\begin{equation}
\sup_{f\in B_M}\|\mathcal T_{\widehat\zeta}f-\mathcal T_\zeta f\|_2
\le
C_{\mathrm{chart}}M\delta.
\end{equation}
\end{lemma}

\begin{proof}
For each $r=1,\dots,R_\xi$, define
\begin{equation}
(T_r^\zeta f)_{ij}
=
a_r(\zeta_{ij})
\Big(\frac1W\sum_{t=1}^W b_r(\zeta_{it})f_{it}\Big),
\qquad
(T_r^{\widehat\zeta} f)_{ij}
=
a_r(\widehat\zeta_{ij})
\Big(\frac1W\sum_{t=1}^W b_r(\widehat\zeta_{it})f_{it}\Big).
\end{equation}
Also define
\begin{equation}
m_r(i;f):=\frac1W\sum_{t=1}^W b_r(\zeta_{it})f_{it},
\qquad
\widehat m_r(i;f):=\frac1W\sum_{t=1}^W b_r(\widehat\zeta_{it})f_{it}.
\end{equation}
Then
\begin{equation}
(T_r^{\widehat\zeta}f-T_r^\zeta f)_{ij}
=
\big(a_r(\widehat\zeta_{ij})-a_r(\zeta_{ij})\big)\widehat m_r(i;f)
+
a_r(\zeta_{ij})\big(\widehat m_r(i;f)-m_r(i;f)\big).
\end{equation}
Since $a_r$ is Lipschitz and $b_r$ is bounded,
\begin{equation}
|a_r(\widehat\zeta_{ij})-a_r(\zeta_{ij})|
\le
L_{a_r}\delta,
\qquad
|\widehat m_r(i;f)|
\le
\frac{B_r}{W}\sum_{t=1}^W |f_{it}|
\le
\frac{B_r}{\sqrt W}\|f_{i,:}\|_2.
\end{equation}
Since $b_r$ is Lipschitz,
\begin{equation}
|b_r(\widehat\zeta_{it})-b_r(\zeta_{it})|
\le
L_{b_r}\delta,
\end{equation}
and therefore
\begin{equation}
\begin{aligned}
|\widehat m_r(i;f)-m_r(i;f)|
&=
\left|
\frac1W\sum_{t=1}^W
\big(b_r(\widehat\zeta_{it})-b_r(\zeta_{it})\big)f_{it}
\right| \\
&\le
\frac{L_{b_r}\delta}{W}\sum_{t=1}^W |f_{it}| \\
&\le
\frac{L_{b_r}\delta}{\sqrt W}\|f_{i,:}\|_2.
\end{aligned}
\end{equation}
Hence
\begin{equation}
|(T_r^{\widehat\zeta}f-T_r^\zeta f)_{ij}|
\le
\frac{\delta}{\sqrt W}(L_{a_r}B_r+A_rL_{b_r})\|f_{i,:}\|_2.
\end{equation}
Squaring and summing over $j$ and then $i$ gives
\begin{equation}
\|T_r^{\widehat\zeta}f-T_r^\zeta f\|_2
\le
\delta(L_{a_r}B_r+A_rL_{b_r})\|f\|_2.
\end{equation}

Similarly, for each $s=1,\dots,R_\eta$, define
\begin{equation}
(S_s^\zeta f)_{ij}
=
c_s(\zeta_{ij})
\Big(\frac1H\sum_{p=1}^H d_s(\zeta_{pj})f_{pj}\Big),
\qquad
(S_s^{\widehat\zeta} f)_{ij}
=
c_s(\widehat\zeta_{ij})
\Big(\frac1H\sum_{p=1}^H d_s(\widehat\zeta_{pj})f_{pj}\Big).
\end{equation}
Repeating the same argument along columns yields
\begin{equation}
\|S_s^{\widehat\zeta}f-S_s^\zeta f\|_2
\le
\delta(L_{c_s}D_s+C_sL_{d_s})\|f\|_2.
\end{equation}

For the local term, define
\begin{equation}
(L^\zeta f)_{ij}:=\ell(\zeta_{ij})f_{ij},
\qquad
(L^{\widehat\zeta}f)_{ij}:=\ell(\widehat\zeta_{ij})f_{ij}.
\end{equation}
Then
\begin{equation}
|(L^{\widehat\zeta}f-L^\zeta f)_{ij}|
=
|\ell(\widehat\zeta_{ij})-\ell(\zeta_{ij})|\,|f_{ij}|
\le
L_\ell\delta |f_{ij}|,
\end{equation}
and thus
\begin{equation}
\|L^{\widehat\zeta}f-L^\zeta f\|_2
\le
L_\ell\delta\|f\|_2.
\end{equation}

Since
\begin{equation}
\mathcal T_\zeta
=
\sum_{r=1}^{R_\xi} T_r^\zeta
+
\sum_{s=1}^{R_\eta} S_s^\zeta
+
L^\zeta,
\qquad
\mathcal T_{\widehat\zeta}
=
\sum_{r=1}^{R_\xi} T_r^{\widehat\zeta}
+
\sum_{s=1}^{R_\eta} S_s^{\widehat\zeta}
+
L^{\widehat\zeta},
\end{equation}
the triangle inequality gives
\begin{equation}
\|\mathcal T_{\widehat\zeta}f-\mathcal T_\zeta f\|_2
\le
\sum_{r=1}^{R_\xi}\|T_r^{\widehat\zeta}f-T_r^\zeta f\|_2
+
\sum_{s=1}^{R_\eta}\|S_s^{\widehat\zeta}f-S_s^\zeta f\|_2
+
\|L^{\widehat\zeta}f-L^\zeta f\|_2.
\end{equation}
Using the bounds above yields
\begin{equation}
\|\mathcal T_{\widehat\zeta}f-\mathcal T_\zeta f\|_2
\le
C_{\mathrm{chart}}\delta\|f\|_2.
\end{equation}
If $f\in B_M$, then $\|f\|_2\le M$, so
\begin{equation}
\|\mathcal T_{\widehat\zeta}f-\mathcal T_\zeta f\|_2
\le
C_{\mathrm{chart}}M\delta.
\end{equation}
Taking the supremum over $f\in B_M$ proves the last claim.
\end{proof}

\begin{theorem}[Approximation of charted axial low-rank operators by one-block CATO]
Let $\widetilde{\mathcal G}_\Phi:B_M\to\mathbb{R}^{H\times W}$
be $(R_\xi,R_\eta,\varepsilon_{\mathrm{rk}})$-charted axial low-rank as defines in Definition~\ref{def}. Then for every $\varepsilon_{\mathrm{nn}}>0$, there
exists a hidden width $C$ and parameters of a one-block core CATO with $R_\xi$
row heads and $R_\eta$ column heads such that
\[
\sup_{f\in B_M}
\|\mathcal N_\Theta(f,X)-\widetilde{\mathcal G}_\Phi f\|_2
\le
\varepsilon_{\mathrm{rk}}M+\varepsilon_{\mathrm{nn}}.
\]
Moreover, if the hypotheses of Lemma~\ref{applem2} hold and
\[
\max_{i,j}\|\widehat\zeta_{ij}-\zeta_{ij}\|\le \delta,
\]
then one can choose a one-block core CATO of the same axial size such that
\[
\sup_{f\in B_M}
\|\mathcal N_\Theta(f,X)-\widetilde{\mathcal G}_\Phi f\|_2
\le
\varepsilon_{\mathrm{rk}}M
+
C_{\mathrm{chart}}M\delta
+
\varepsilon_{\mathrm{nn}}.
\]
\end{theorem}

\begin{proof}
By Definition~1,
\begin{equation}
\widetilde{\mathcal G}_\Phi=\mathcal T_\zeta+\mathcal R,
\qquad
\|\mathcal R f\|_2\le \varepsilon_{\mathrm{rk}}\|f\|_2
\quad
\text{for all }f\in B_M.
\end{equation}

For the first claim, Lemma~\ref{applemma} implies that for every
$\varepsilon_{\mathrm{nn}}>0$ there exists a hidden width $C$ and parameters of
a one-block core CATO such that
\begin{equation}
\sup_{f\in B_M}\|\mathcal N_\Theta(f,X)-\mathcal T_\zeta f\|_2
\le
\varepsilon_{\mathrm{nn}}.
\end{equation}
Therefore, for every $f\in B_M$,
\begin{equation}
\begin{aligned}
\|\mathcal N_\Theta(f,X)-\widetilde{\mathcal G}_\Phi f\|_2
&\le
\|\mathcal N_\Theta(f,X)-\mathcal T_\zeta f\|_2
+
\|\mathcal R f\|_2 \\
&\le
\varepsilon_{\mathrm{nn}}
+
\varepsilon_{\mathrm{rk}}\|f\|_2 \\
&\le
\varepsilon_{\mathrm{nn}}+\varepsilon_{\mathrm{rk}}M.
\end{aligned}
\end{equation}
Taking the supremum over $f\in B_M$ gives
\begin{equation}
\sup_{f\in B_M}
\|\mathcal N_\Theta(f,X)-\widetilde{\mathcal G}_\Phi f\|_2
\le
\varepsilon_{\mathrm{rk}}M+\varepsilon_{\mathrm{nn}}.
\end{equation}

For the second claim, let $\mathcal T_{\widehat\zeta}$ be obtained from
$\mathcal T_\zeta$ by replacing $\zeta$ with $\widehat\zeta$. Applying
Lemma~\ref{applemma} to $\mathcal T_{\widehat\zeta}$ yields a
one-block core CATO such that
\begin{equation}
\sup_{f\in B_M}\|\mathcal N_\Theta(f,X)-\mathcal T_{\widehat\zeta}f\|_2
\le
\varepsilon_{\mathrm{nn}}.
\end{equation}
Then for every $f\in B_M$,
\begin{equation}
\begin{aligned}
\|\mathcal N_\Theta(f,X)-\widetilde{\mathcal G}_\Phi f\|_2
&\le
\|\mathcal N_\Theta(f,X)-\mathcal T_{\widehat\zeta}f\|_2
+
\|\mathcal T_{\widehat\zeta}f-\mathcal T_\zeta f\|_2
+
\|\mathcal R f\|_2.
\end{aligned}
\end{equation}
By Lemma~\ref{applem2},
\begin{equation}
\|\mathcal T_{\widehat\zeta}f-\mathcal T_\zeta f\|_2
\le
C_{\mathrm{chart}}\delta\|f\|_2
\le
C_{\mathrm{chart}}M\delta,
\end{equation}
and by Definition~1,
\begin{equation}
\|\mathcal R f\|_2\le \varepsilon_{\mathrm{rk}}\|f\|_2\le \varepsilon_{\mathrm{rk}}M.
\end{equation}
Therefore
\[
\|\mathcal N_\Theta(f,X)-\widetilde{\mathcal G}_\Phi f\|_2
\le
\varepsilon_{\mathrm{nn}}+C_{\mathrm{chart}}M\delta+\varepsilon_{\mathrm{rk}}M.
\]
Taking the supremum over $f\in B_M$ yields
\[
\sup_{f\in B_M}
\|\mathcal N_\Theta(f,X)-\widetilde{\mathcal G}_\Phi f\|_2
\le
\varepsilon_{\mathrm{rk}}M
+
C_{\mathrm{chart}}M\delta
+
\varepsilon_{\mathrm{nn}}.
\]
This completes the proof.
\end{proof}

\section{Benchmarks Details}
\label{Benchmark}
In this section, we provide a summary of the dataset and the details of each dataset.  In Table \ref{datasummary}, we provide the details of different types of PDEs. Then we provide the formulation of different PDEs.

\paragraph{Plasticity} This benchmark evaluates a model's ability to predict the future deformation of a plastic material subjected to impact from an arbitrarily shaped die applied from above \cite{li2023fourier}. In each case, the input is the die geometry, discretized on a structured mesh and represented as a tensor of size $101 \times 31$. The target output is the deformation field at each mesh point over the next 20 time steps. This output is represented as a tensor of size $20 \times 101 \times 31 \times 4$, where the final dimension corresponds to deformation components in four directions. The dataset contains 900 samples with distinct die shapes for training and 80 additional samples for testing.

\paragraph{Airfoil} This benchmark focuses on predicting the Mach number field induced by different airfoil geometries, following \cite{li2023fourier}. Each airfoil shape is represented on a structured mesh of size $221 \times 51$, and the target output is the Mach number evaluated at every mesh point. All airfoil geometries are generated by deforming the baseline NACA-0012 profile provided by the National Advisory Committee for Aeronautics. In total, 1,000 airfoil designs are used for training, while an additional 200 samples are reserved for testing.

\paragraph{Pipe} This benchmark considers the prediction of the horizontal fluid velocity field from the geometry of a pipe, following \cite{li2023fourier}. For each sample, the pipe domain is represented using a structured mesh of size $129 \times 129$. The input is therefore a tensor of size $129 \times 129 \times 2$, where the last dimension stores the two-dimensional coordinates of each mesh point. The target output is the horizontal velocity value at every mesh location, represented as a tensor of size $129 \times 129 \times 1$. The dataset contains 1,000 pipe geometries for training and 200 additional geometries for testing, generated by varying the pipe centerline.

\paragraph{Navier-Stokes} This benchmark studies the prediction of incompressible viscous fluid dynamics on a unit torus, following \cite{li2020fourier}. The fluid is assumed to have constant density, with the viscosity fixed at $10^{-5}$. The velocity field is discretized on a regular grid of size $64 \times 64$. Given the flow observations from the previous 10 time steps, the task is to forecast the fluid evolution over the next 10 time steps. The dataset consists of 1,000 fluid trajectories with different initial conditions for training, together with 200 additional trajectories for testing.

\paragraph{Darcy} This benchmark evaluates the modeling of fluid flow through porous media, following \cite{li2020fourier}. The original simulation domain is discretized on a regular grid of size $421 \times 421$, which is downsampled to $85 \times 85$ for the main experiments. For each sample, the model takes the porous medium structure as input and predicts the corresponding pressure field over the grid. The dataset includes 1,000 training samples with varying medium structures, and an additional 200 for testing.

\paragraph{Elasticity} This benchmark investigates the prediction of internal stress fields in elastic materials from their underlying structural geometry, following \cite{li2023fourier}. Each material sample is represented by 972 discretized points. The model input is a tensor of size $972 \times 2$, where each row encodes the two-dimensional coordinates of a point. The target output is the corresponding stress value at each point, represented as a tensor of size $972 \times 1$. The dataset contains 1,000 material structures for training and 200 additional structures for testing.
\begin{table*}[t]
\centering
\caption{Benchmark datasets used in the experiments. Here, $N$ denotes the spatial resolution and $N_t$ denotes the temporal dimension.}
\label{tab_dataset_detail}
\setlength{\tabcolsep}{5.5pt}
\renewcommand{\arraystretch}{1.15}

\begin{tabular}{lllp{2.8cm}ccc}
\toprule
\textbf{Type} 
& \textbf{Benchmark} 
& \textbf{Geometry} 
& \textbf{Task: input $\rightarrow$ output} 
& \textbf{$N$} 
& \textbf{$N_t$} 
& \textbf{Train/Test} \\
\midrule

\multirow{2}{*}{Regular grid}
& Darcy 
& Grid 
& Diffusion coefficient $\rightarrow$ fluid pressure 
& $85 \times 85$ 
& -- 
& $1000/200$ \\

& NS 
& Grid 
& Past velocity $\rightarrow$ future velocity 
& $64 \times 64$ 
& $10$ 
& $1000/200$ \\

\midrule

\multirow{3}{*}{Structured mesh}
& Airfoil 
& Mesh 
& Mesh points $\rightarrow$ Mach number 
& $221 \times 51$ 
& -- 
& $1000/200$ \\

& Pipe 
& Mesh 
& Mesh points $\rightarrow$ fluid velocity 
& $129 \times 129$ 
& -- 
& $1000/200$ \\

& Plasticity 
& Mesh 
& Mesh points $\rightarrow$ mesh deformation 
& $101 \times 31$ 
& $20$ 
& $900/80$ \\

\midrule

Point cloud
& Elasticity 
& Cloud 
& Structure $\rightarrow$ inner stress 
& $972$ 
& -- 
& $1000/200$ \\

\bottomrule
\end{tabular}
\label{datasummary}
\end{table*}

\section{Implementation details}
\label{implmentation}
In this section, we provide an overview of the experiment setup, the hyperparameters of our method, the baselines, and the evaluation metrics.

\begin{table}[t]
    \centering
    \caption{Training configurations used by all baselines. Training settings follow previous work without extra tuning. For Darcy, an additional spatial gradient regularization term \(l_{\mathrm{gdl}}\) is adopted following ONO.}
    \label{tab:training_config}
    \small
    \renewcommand{\arraystretch}{1.2}
    \setlength{\tabcolsep}{5pt}

    \begin{tabular}{lcccccc}
        \toprule
        \textbf{Benchmark} 
        & \textbf{Loss} 
        & \textbf{Epochs} 
        & \textbf{LR} 
        & \textbf{Optimizer} 
        & \textbf{Batch} 
        & \textbf{Scheduler} \\
        \midrule
        Darcy 
        & \(l_2 + 0.1l_{\mathrm{gdl}}\) 
        & 500 
        & \(5 \times 10^{-4}\) 
        & AdamW 
        & 4 
        & OneCycleLR \\

        Navier--Stokes 
        & Rel. \(L^2\) 
        & 500 
        & \(5 \times 10^{-4}\) 
        & AdamW 
        & 2 
        & OneCycleLR \\

        Elasticity 
        & Rel. \(L^2\) 
        & 500 
        & \(10^{-3}\) 
        & AdamW 
        & 1 
        & OneCycleLR \\

        Plasticity 
        & Rel. \(L^2\) 
        & 500 
        & \(10^{-3}\) 
        & AdamW 
        & 8 
        & OneCycleLR \\

        Airfoil 
        & Rel. \(L^2\) 
        & 500 
        & \(10^{-3}\) 
        & AdamW 
        & 4 
        & OneCycleLR \\

        Pipe 
        & Rel. \(L^2\) 
        & 500 
        & \(10^{-3}\) 
        & AdamW 
        & 4 
        & OneCycleLR \\
        \bottomrule
    \end{tabular}
    \label{raining configurations}
\end{table}

\begin{table}[t]
    \centering
    \caption{Architecture configurations used by our method across benchmarks.}
    \label{tab:architecture_config}
    \small
    \renewcommand{\arraystretch}{1.2}
    \setlength{\tabcolsep}{6pt}

    \begin{tabular}{lcccccc}
        \toprule
        \textbf{Benchmark} 
        & \textbf{Layers} 
        & \textbf{Embed. Dim} 
        & \textbf{Heads} 
        & \textbf{Grad weight} 
        & \textbf{Flux weight} 
        & \textbf{Consist weight} 
        \\
        \midrule
        Darcy 
        & 8 & 96 & 8 & 0.2 & 0.2 & 0.05 \\

        Navier--Stokes 
        & 8 & 128 & 8 & 0 & 0 & 0 \\

        Elasticity 
        & 8 & 144 & 8 & 0 & 0 & 0 \\

        Plasticity 
        & 8 & 160 & 8 & 0 & 0 & 0 \\

        Airfoil 
        & 8 & 128 & 8 & 0.2 & 0.2 & 0.05 \\

        Pipe 
        & 8 & 96 & 8 & 0.2 & 0.2 & 0.05 \\
        \bottomrule
    \end{tabular}
    \label{architecture}
\end{table}

\subsection{Training Details}
Table \ref{tab_dataset_detail} provides a detailed summary of the data geometry, task, and numbers of training and testing samples. Table
\ref{raining configurations} provides the training configuration used for all baselines.   It summarizes the training configurations used for different methods across the benchmark datasets. To ensure a fair comparison, all baselines and benchmarks are trained under consistent settings, with our method using fewer or comparable parameters than transformer-based baselines. Across all datasets, training employs a relative $\ell_2$ loss. For the Darcy benchmark, following ONO \cite{xiao2023improved}, an additional spatial gradient regularization term $\ell_{\mathrm{gdl}}$ is included, yielding the objective
\begin{equation}
\ell_2 + 0.1\,\ell_{\mathrm{gdl}}.
\end{equation}
All models are trained for 500 epochs using the AdamW optimizer, with the learning rate scheduled using OneCycleLR.

\subsection{Hyperparameters and architecture details}
As shown in Table \ref{architecture}, we set the number of layers and heads to 8, consistent with Transolver and SAOT. In addition, we apply the physical loss to the Darcy, Airfoil, and Pipe models, which are time-independent PDEs. Grad weight is used to weight the gradient-matching loss between the approximated gradient and the true gradient. A larger value encourages the predicted solution to have more accurate spatial derivatives. Flux weight means the weight of the flux loss between the predicted flux and the true gradient. A larger value encourages the predicted flux field to directly match the target physical gradient. Consistent weight means the consistency loss between the predicted flux and the predicted gradient. A larger value encourages the predicted flux to be consistent with the predicted solution itself.

\subsection{Evaluation Metric}
\label{metric}

To evaluate predictive accuracy on standard partial differential equation (PDE) benchmarks, we adopt the mean relative $\ell_2$ error \cite{li2020fourier} as the primary performance measure. This metric is widely used for assessing the discrepancy between predicted and reference physical fields and is reported consistently across all experiments. Formally, the evaluation loss is defined as
\begin{equation}
\mathcal{L}
=
\frac{1}{N}
\sum_{i=1}^{N}
\frac{
\left\|
\mathcal{G}_{\theta}(\mathbf{a}_i)
-
\mathcal{G}^{\dagger}(\mathbf{a}_i)
\right\|_2
}{
\left\|
\mathcal{G}^{\dagger}(\mathbf{a}_i)
\right\|_2
},
\end{equation}
where $N$ denotes the number of test samples, $\mathcal{G}_{\theta}(\mathbf{a}_i)$ is the model prediction corresponding to the input $\mathbf{a}_i$, and $\mathcal{G}^{\dagger}(\mathbf{a}_i)$ represents the associated ground-truth solution. The normalization by $\left\|\mathcal{G}^{\dagger}(\mathbf{a}_i)\right\|_2$ accounts for differences in the magnitude and resolution scale of the target fields, thereby enabling a fair and comparable assessment across heterogeneous PDE benchmarks.
\subsection{CATO-PC}
\label{catopc}

For point-cloud inputs, the row--column factorization required by charted axial attention is unavailable. We therefore introduce CATO-PC, an irregular-mesh variant that retains the learned chart
\(\zeta_i=\Phi_{\mathrm{chart}}(x_i)\) but replaces structured axial attention with a combination of irregular physics attention from \cite{wu2024transolver} and local chart-conditioned message passing. Given an unordered point set \(X=\{x_i\}_{i=1}^N\), optional features \(F=\{f_i\}_{i=1}^N\), and chart coordinates \(\zeta_i\in[-1,1]^{d_\zeta}\), the input token is lifted as
\[
    h_i^{(0)}
    =
    \Phi_{\mathrm{pre}}
    \bigl([\rho(x_i),f_i,\zeta_i]\bigr)
    +
    \Phi_{\mathrm{cb}}(\zeta_i),
\]
where \(f_i\) is omitted when no auxiliary feature is provided. A \(K\)-nearest-neighbor graph is constructed in the physical coordinate space. For each edge \((i,j)\), we define
\[
    g_{ij}
    =
    [x_j-x_i,\|x_j-x_i\|_2,\zeta_j-\zeta_i],
\]
and compute local messages
\[
    m_{ij}
    =
    \sigma\!\left(
        W_c\bar h_i
        +
        W_\Delta(\bar h_j-\bar h_i)
        +
        \Phi_{\mathrm{geo}}(g_{ij})
    \right),
    \qquad
    \bar h_i=\operatorname{LN}(h_i).
\]
The local operator aggregates messages by both soft attention and max pooling:
\[
    \mathcal{L}_{\mathrm{pc}}(H,X,\zeta)_i
    =
    \Phi_{\mathrm{out}}
    \left(
    \left[
        \sum_{j\in\mathcal{N}_K(i)}
        \alpha_{ij}m_{ij},
        \;
        \max_{j\in\mathcal{N}_K(i)}m_{ij}
    \right]
    \right),
\]
where
\[
    \alpha_{ij}
    =
    \operatorname{softmax}_{j\in\mathcal{N}_K(i)}
    \left(
        \frac{w_s^\top m_{ij}}{\sqrt{C}}
    \right).
\]
Each block then updates the hidden state by
\[
    H^{(\ell,1)}
    =
    H^{(\ell)}
    +
    \gamma_{\mathrm{attn}}
    \odot
    \mathcal{A}_{\mathrm{irr}}
    \bigl(\operatorname{LN}(H^{(\ell)})\bigr),
\]
\[
    H^{(\ell,2)}
    =
    H^{(\ell,1)}
    +
    \gamma_{\mathrm{loc}}
    \odot
    \mathcal{L}_{\mathrm{pc}}
    \bigl(\operatorname{LN}(H^{(\ell,1)}),X,\zeta\bigr),
\]
\[
    H^{(\ell+1)}
    =
    H^{(\ell,2)}
    +
    \gamma_{\mathrm{mlp}}
    \odot
    \operatorname{MLP}
    \bigl(\operatorname{LN}(H^{(\ell,2)})\bigr).
\]
The final representation is mapped to the solution prediction
\(\hat u_i=\Phi_u(\bar h_i)\), and optionally to an auxiliary flux-like field
\(\hat q_i=\Phi_q(\bar h_i)\). In this way, CATO-PC preserves the learned chart mechanism of CATO while adding topology-aware local interactions suitable for irregular meshes and unordered point clouds.

\section{More visualization and ablation study}
\label{More abaltion}
In this section, we provide more ablation studies and visualization. 

\begin{figure}[t!]
    \centering
    \includegraphics[width=1\linewidth]{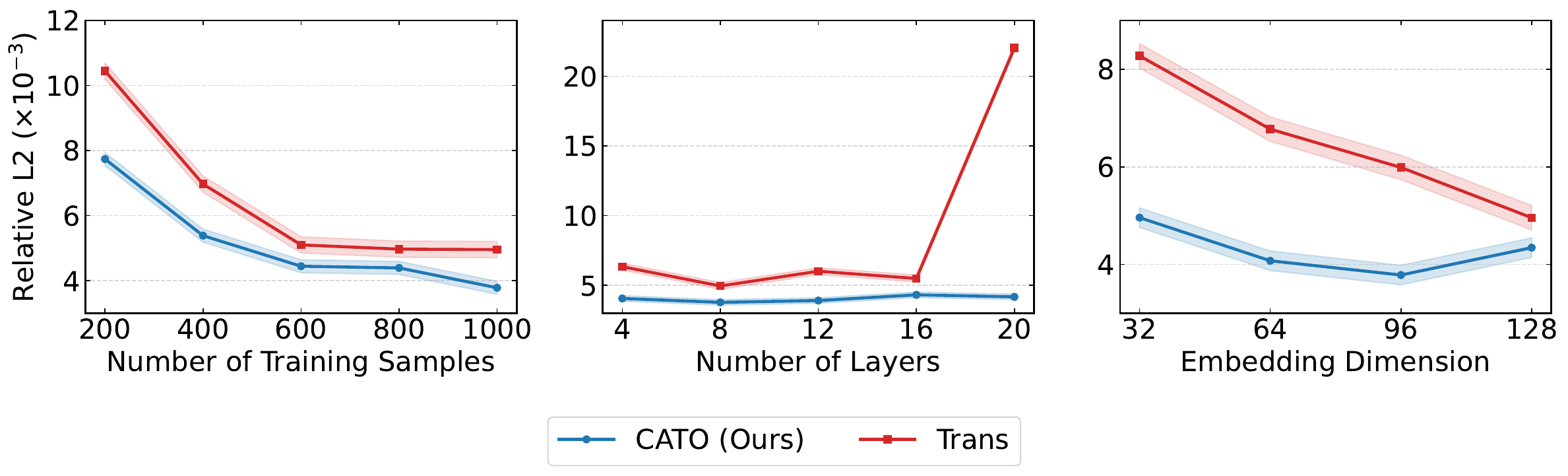}
    \caption{Model scaling performance on Pipe. We compare our method with Transolver across training sample size, layer count, and embedding dimension. }
    
\end{figure}

\begin{figure}[t!]
    \centering
    \includegraphics[width=1\linewidth]{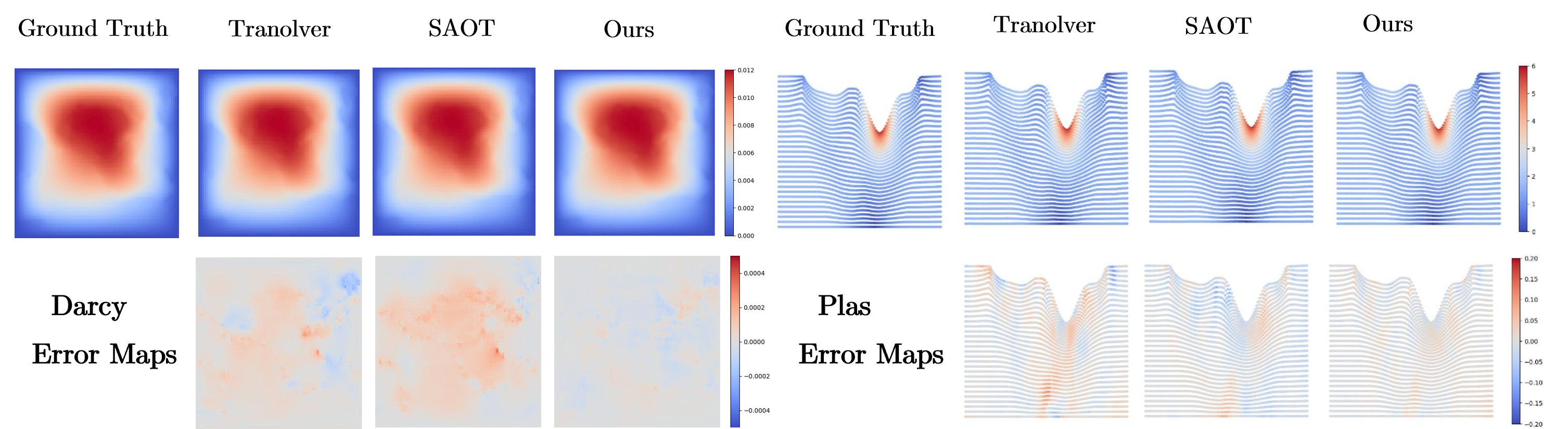}
    \caption{Visual comparison on Darcy and Plas benchmarks. The top row shows the ground truth and predictions from Transolver, SAOT, and our method. The bottom row presents the corresponding error maps for each prediction method. }

\end{figure}

\begin{figure}[t!]
    \centering
    \includegraphics[width=1\linewidth]{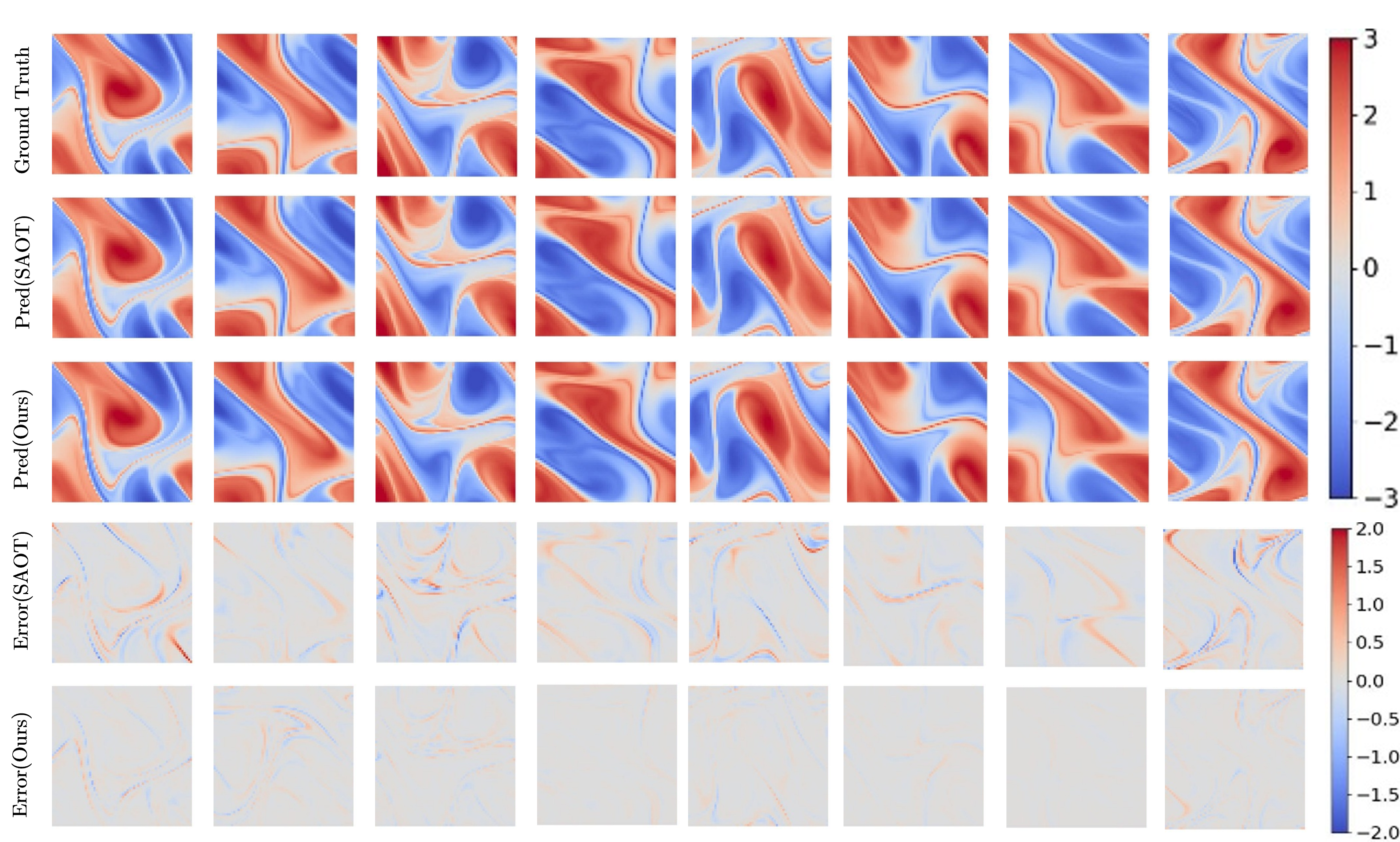}
    \caption{Teaser visualization on the Navier–Stokes benchmark.
Comparison of ground truth, SAOT prediction, CATO prediction, and their corresponding error maps across multiple test cases. CATO produces predictions closer to the ground truth and yields smaller, more localized errors than SAOT, indicating improved accuracy in capturing complex flow structures.}
\end{figure}

\section{Broad Impact} 
\label{broadimpact}
This work introduces CATO, a deep learning-based solver with broad applicability across scientific and engineering problems. Although CATO is not designed for social-domain applications such as large language models or image generation, its computational capabilities may benefit a wide range of real-world settings, including weather forecasting, biomedical imaging, industrial simulation, and engineering optimization. Its broader impact lies in enabling more efficient, scalable, and accurate computational modeling for applications with significant scientific, industrial, and societal relevance.


\clearpage
\clearpage
\newpage

\end{document}